\documentclass[11pt]{article}
\usepackage[margin=1in]{geometry}
\usepackage{authblk} 
\usepackage[T1]{fontenc}
\usepackage{lmodern} % ensure scalable fonts for microtype expansion
\makeatletter
\providecommand\IfPDFManagementActiveTF{\@secondoftwo}
\providecommand\IfPDFManagementActiveT{\@gobble}
\providecommand\IfPDFManagementActiveF{\@firstofone}
\makeatother
\usepackage{amsmath}
\usepackage{amssymb}
\usepackage{amsthm}
\usepackage[expansion=false]{microtype}
\usepackage[hidelinks]{hyperref}
\usepackage{cleveref} 
\usepackage{booktabs}

\usepackage{url}
\usepackage{placeins}
\usepackage{graphicx}
\usepackage{float}
\usepackage{thmtools}
\usepackage[draft]{svg}
\usepackage{thm-restate}
\usepackage{stmaryrd}
\usepackage{xspace}
\usepackage{changepage}
\usepackage{paralist}
\usepackage{multirow}
\usepackage{subcaption}
\usepackage{longtable}
\usepackage{newfile}
\usepackage{refcount}
\usepackage{tikz}

\crefname{item}{RQ}{RQ}

\usepackage[mathscr]{euscript}

\usepackage{pgf}
\usepackage{pgfplots}
\pgfplotsset{compat=1.18}
\usepackage{tikz}
\usetikzlibrary{shapes,shapes.geometric,fit,trees,decorations,arrows,automata,shadows,positioning,plotmarks,backgrounds,tikzmark}
\usetikzlibrary{calc,matrix,fit,petri,decorations.markings,decorations.pathmorphing,patterns,intersections,decorations.text}

\usepackage{amscd}
\usepackage{array}
\usepackage{xifthen}

\usepackage{interval}

\usepackage{stmaryrd}
\usepackage{pifont}
\usepackage{fontawesome}
\usepackage{subcaption}
\usepackage{tikz}
\usetikzlibrary{patterns, automata, positioning, arrows}

\providecommand{\keywords}[1]{\par\medskip\noindent\textbf{Keywords:} #1}

\makeatletter
\@ifundefined{problem}{\newtheorem{problem}{Problem}}{}
\@ifundefined{lemma}{\newtheorem{lemma}{Lemma}}{}
\@ifundefined{claim}{}{}
\@ifundefined{theorem}{\newtheorem{theorem}{Theorem}}{}
\@ifundefined{definition}{}{}
\@ifundefined{example}{\newtheorem{example}{Example}}{}
\makeatother

\newcommand{\real}{\mathbb{R}}

\newcommand{\agent}{\protect$\triangle$}

\newcommand{\prop}{\mathrm{AP}}

\newcommand{\aut}{\mathcal{A}}

\newcommand{\xup}[2][]{\mathrel{\rotatebox[origin=c]{90}{$\xrightarrow[#1]{#2}$}}}
\newcommand{\xdown}[2][]{\mathrel{\rotatebox[origin=c]{-90}{$\xrightarrow[#1]{#2}$}}}

\newcommand{\xright}[2][]{\xrightarrow[#1]{#2}}

\newcommand{\mdp}{\mathcal{M}}
\newcommand{\region}{R}
\newcommand{\corner}{\textit{C}}
\newcommand{\regionset}{\mathcal{R}}
\newcommand{\cornerset}{\mathcal{C}}

\begin{document}

\title{About Time: Model-free Reinforcement Learning with Timed Reward Machines}
% add dummy authors with affiliations
\author{
  Rajarshi Roy$^{1}$\footnote{This work was partly carried out while the author was at University of Liverpool, UK.},
  Anirban Majumdar$^{2}$,
  Ritam Raha$^{3}$,\\ 
  David Parker$^{1}$, and
  Marta Kwiatkowska$^{1}$
}

\affil{
$^{1}$Department of Computer Science, University of Oxford, UK \\
$^{2}$Tata Institute of Fundamental Research, Mumbai, India\\
$^{3}$Max Planck Institute for Software Systems, Kaiserslautern, Germany
}

\date{}

\maketitle

\begin{abstract}
Reward specification plays a central role in reinforcement learning (RL), guiding the agent’s behavior. 
To express non-Markovian rewards, formalisms such as reward machines have been introduced to capture dependencies on histories. 
However, traditional reward machines lack the ability to model precise timing constraints, limiting their use in time-sensitive applications. 
In this paper, we propose timed reward machines (TRMs), which are an extension of reward machines that incorporate timing constraints into the reward structure. TRMs enable more expressive specifications with tunable reward logic, for example, imposing costs for delays and granting rewards for timely actions.
We study model-free RL frameworks (i.e., tabular Q-learning) for learning optimal policies with TRMs under digital and real-time semantics.
Our algorithms integrate the TRM into learning via abstractions of timed automata, and employ counterfactual-imagining heuristics that exploit the structure of the TRM to improve the search.
Experimentally, we demonstrate that our algorithm learns policies that achieve high rewards while satisfying the timing constraints specified by the TRM on popular RL benchmarks. 
Moreover, we conduct comparative studies of performance under different TRM semantics, along with ablations that highlight the benefits of counterfactual-imagining.
\end{abstract}


\section{Introduction}
Reinforcement Learning (RL)~\cite{DBLP:books/lib/SuttonB2018} has become a foundational paradigm for sequential decision-making, enabling agents to learn optimal behavior through interactions with an environment. A crucial aspect of any RL problem is the reward specification, which defines the agent's learning objective. Traditionally, rewards are assumed to depend only on the current state and action, conforming to the \emph{Markov} property. However, many real-world tasks require objectives that depend on the history of states and actions, such as completing a sequence of tasks or avoiding repeated errors. To address this, \emph{non-Markovian} reward formalisms have been developed, with reward machines (RMs) emerging as a prominent approach.

Reward Machines~\cite{DBLP:journals/jair/IcarteKVM22} are finite-state automata that specify structured, history-dependent reward functions. They provide a compact and expressive way to encode high-level objectives and have been successfully integrated into RL frameworks, improving sample efficiency and interpretability. However, a critical limitation of existing RMs is their inability to express timing constraints---a vital requirement to specify time-sensitive requirements in domains such as robotics and autonomous driving~\cite{MEHDIPOUR2023110692,DBLP:journals/arobots/SadighLSSD18}. For instance, an AV might need to ``slow down for 3 seconds to allow a pedestrian to cross'' or ``avoid an unsafe road for at least 10 seconds''.

In this paper, we propose timed reward machines (TRMs) that enhance reward machines with fine-grained timing requirements. To this end, we augment reward machines with clocks that track elapsed time and use them to impose timing constraints, drawing inspiration from the timed automata (TA) literature~\cite{DBLP:journals/tcs/AlurD94}. TRMs thus allow reward functions to depend not only on the agent’s history of states and actions, but also on the time intervals between events. Moreover, TRMs can assign costs and rewards to both states and transitions, incentivizing the agent to complete a task in a timely manner while optimizing the overall reward.
% To illustrate our problem, we consider a simple stochastic office world environment, adapted from~\cite{DBLP:conf/ijcai/CamachoIKVM19}, as shown in Fig.~\ref{fig:office-world}.
% The environment features three key elements: office \office{}, coffee \coffee{}, and decoration \decor{}. The primary objective is to reach the office within 2 time units of collecting coffee in order to keep it warm, which yields a reward of 10. If this is not achieved, the agent can instead reach the office directly within 5 time units for a reduced reward of 2. Additionally, if the agent encounters a decoration at any state, it incurs a penalty of -10 and is immediately sent to the sink.

\begin{example}
To illustrate our setting, we define a simple TRM (Fig.~\ref{fig:reward-machine}) on the standard Taxi domain (Fig.~\ref{fig:intro-taxi}). The TRM encourages a taxi agent to drive slowly, e.g., due to heavy traffic, by providing a higher reward when it delays at each step (enforced by a self-loop with timing constraint $x>1$). Additionally, it imposes a deadline for picking up the passenger (enforced by a transition with timing constraint $y\leq 14$). Finally, after pickup, the agent must reach the destination while driving slowly. Such time-sensitive objectives, involving delays and deadlines, can be naturally captured by TRMs.
\end{example}

\begin{figure*}
    \centering
    %subfigure1
    \begin{subfigure}{1\textwidth}
        \centering     \includegraphics[width=0.5\linewidth,height=0.5\textheight,keepaspectratio]{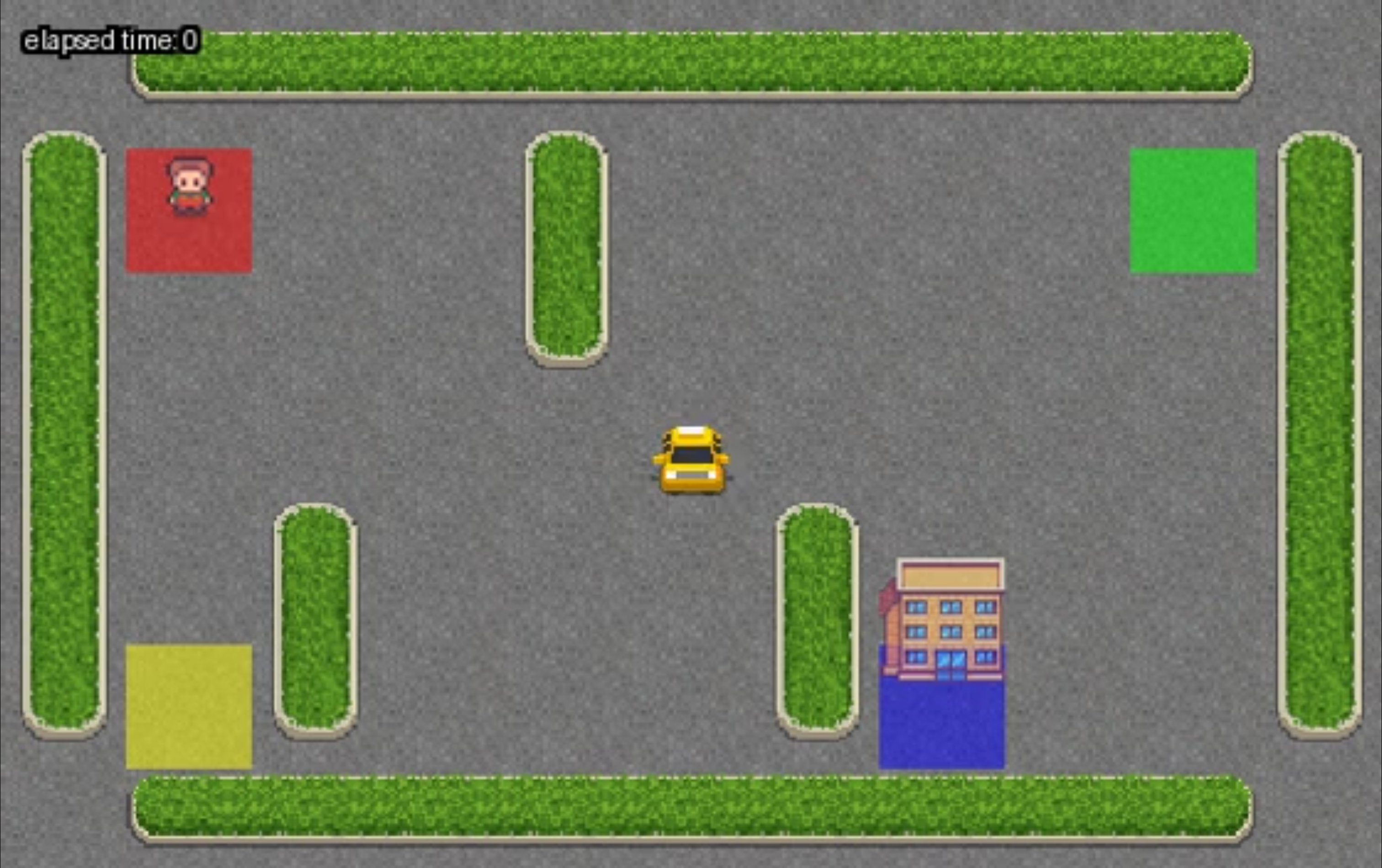}
        \caption{Taxi Domain}
        \label{fig:intro-taxi}
        \vspace{2mm}
    \end{subfigure}
    \hskip 0.5cm
    %subfigure2
    \begin{subfigure}{1\textwidth}
        \centering
        \begin{tikzpicture}[->, >=stealth, node distance=cm, on grid, auto, font=\scriptsize, initial text=]
    \node[state, initial, rectangle, rounded corners] (u0) at (0,0) {$u_0,-30$};
    \node[state,rectangle, rounded corners] (u1) at (4.5,0) {$u_1,-30$};
    \node[state,rectangle, rounded corners, fill=gray!10] (u2) at (7.5,0) {$u_2$};

    \path (u0) edge[loop above] node{$\{\},\ x>1,\{x\}, -10$} (u0)
          (u0) edge[loop below] node{$\{\},\ x\le 1,\{x\}, -50$} (u0)
          (u0) edge node{\texttt{pick\_pass},\,500,\,$y\leq 14$} (u1)
          (u1) edge[loop above] node{$\{\},\ x>1,\{x\}, -10$} (u1)
          (u1) edge[loop below] node{$\{\},\ x\le 1,\{x\}, -50$} (u1)
          (u1) edge node{\texttt{at\_dest},\,800} (u2);
  \end{tikzpicture}
        \caption{TRM example}
        \label{fig:reward-machine}
        \vspace{2mm}
    \end{subfigure}
    \\
    \vskip 0.5cm
    \begin{subfigure}{1\textwidth}
    \centering
    \begin{tikzpicture}
        \node[anchor=south west, inner sep=0] (img) at (0,0) {\includegraphics[width=\linewidth,height=0.27\textheight,keepaspectratio]{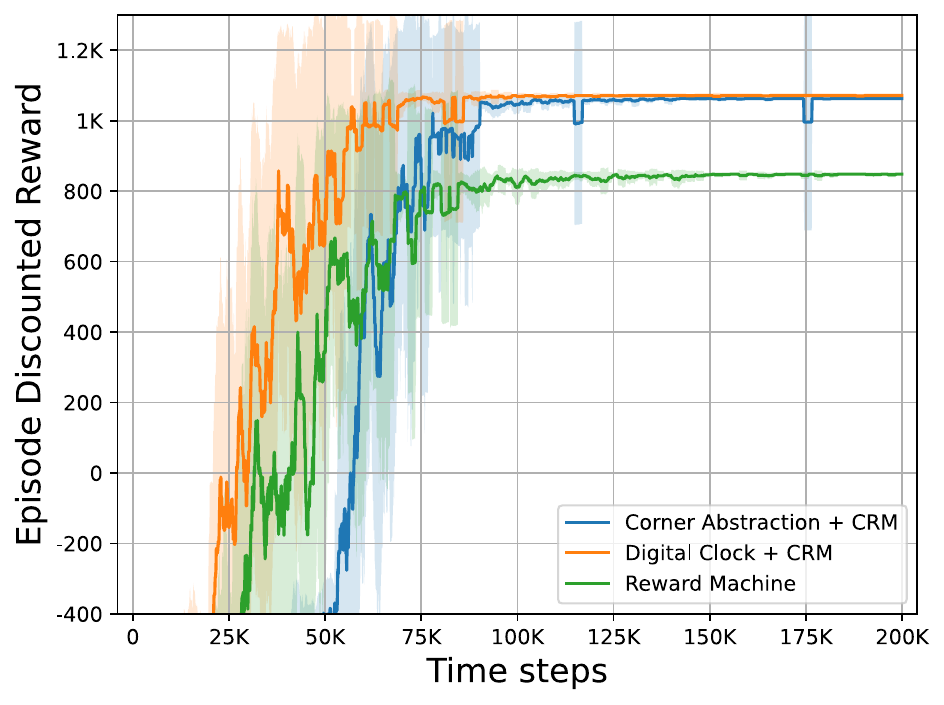}};
        % Add a white textbox with text at desired position (e.g., top left)
        \node[fill=white, opacity=1, text opacity=1, font=\fontsize{5}{8}\selectfont, align=center, inner sep=0pt] at (6.65,1.34) {Corner Abstraction+CI \\ Digital Clock+CI \\ Reward Machine};
    \end{tikzpicture}
    \caption{Reward comparison for presented RL algorithms}
    \label{fig:intro-graph}
\end{subfigure}
    
\caption{An illustration of TRM on Gym Taxi domain: (a) Taxi Domain example, with a passenger in location red and destination in location blue, (b) A TRM that instructs the taxi to pick up a passenger and drop her at a destination, while moving slowly, and (c) Rewards obtained using digital and real-time TRM, and reward machines.}~\label{fig:intro}
\end{figure*}

We interpret TRMs over Markov decision processes (MDPs) to model stochastic environments (Section~\ref{sec:problem}). To express timing constraints, we augment the MDP action set with explicit delay actions. We then study two standard timing interpretations---digital-clock (Section~\ref{sec:tabularQ-digital}) and real-time (Section~\ref{sec:tabularQ-cont})---also known as integer-clock and dense-time interpretations in TA literature~\cite{DBLP:conf/icalp/HenzingerMP92}. For each setting, we devise tabular Q-learning algorithms on product MDPs obtained by integrating the TRM in the environment.

In the digital-clock case, the product construction is straightforward: integer clock valuations are directly included in the MDP state space. In the real-time case, we consider two approaches: (i) discretized time for continuous delays, and (ii) a corner-point abstraction based on region construction of timed automata, which encourages agents to pick delays that are very close to integers.
% as close as possible to integer values.
% 
We further enhance these algorithms with novel counterfactual imagining (CI), inspired by~\cite{DBLP:journals/jair/IcarteKVM22}, to incorporate alternative clock valuations and timing delays.
We also provide a detailed theoretical analysis of all the presented settings, including guarantees of convergence to ($\varepsilon$-)optimal policies.

Our algorithms based on TRMs, as opposed to standard reward machines (without delay actions), can enforce timing constraints and therefore obtain higher rewards that may depend on precise timing.
We demonstrate this in Fig.~\ref{fig:intro-graph} for our running example under both digital and real-time interpretations.
% \st{Notably, standard reward machines (without delay actions) cannot enforce timing constraints and therefore miss higher rewards that depend on precise timing, whereas TRMs can enable learning these.}
For detailed experiments, we interpret several TRMs over standard RL benchmarks (Section~\ref{sec:experiments}). Our experiments reveal clear differences between the digital and real-time settings: in scenarios requiring substantial delays, the corner-point abstraction often yields better returns. Moreover, we consistently observe higher returns from our CI heuristics.

%All proofs and additional experimental results can be found in the supplementary material.
% In this paper, we present the following contributions:
% \begin{itemize}
%     \item We introduce Timed Reward Machines (TRMs), a formalism that extends reward machines to incorporate timing constraints, enabling the specification of time-sensitive objectives in RL tasks.
%     \item We develop a model-free RL algorithm for learning optimal policies with TRMs, interpreting the time as both digital and real-time settings. The model-free algorithm requires 
%     We show that standard procedures for reward machines, and counterfactual imagining~\cite[Section 3.3]{DBLP:phd/ca/Icarte22}, extend to TRMs.
%     \item In the discrete time setting, we can adapt the standard time setting of Q-learning on a product space of the MDP environment and TRM. In the continuous time setting, we explore abstractions of TRM, such as region and zone-based abstractions for constructing the product space, and adapt standard deep RL methods such as Deep Q-Networks (DQN)~\cite{DBLP:journals/nature/MnihKSRVBGRFOPB15}/Deep Deterministic Policy Gradient (DDPG)~\cite{DBLP:journals/corr/LillicrapHPHETS15} to learn optimal policies.
%     \item We evaluate the performance of TRMs in various RL environments, demonstrating their effectiveness in capturing time-sensitive objectives and improving learning efficiency compared to traditional reward machines.
% \end{itemize}

\subsection{Related Work.} We contrast our work with existing approaches, presenting them in order of relevance. For more detailed technical comparison with similar works, see Section~\ref{sec:technical-comparison-related-works}.

\paragraph{RL with Timed Specifications.} To our knowledge, only a handful of works consider time-sensitive logical reward specifications for RL.~\cite{DBLP:conf/ijcai/0005T19} optimize Metric Temporal Logic (MTL) objectives by translating them into (simple) timed automata, considering only the digital-clock setting.~\cite{DBLP:conf/rtss/DoleGKKT21,DBLP:conf/aaai/DoleGKK023} study subclasses of Duration Calculus that compile to variants of timed automata. Both lines of work rely on timed automata as monitors (reward 1 for accept, 0 for reject) to maximize \emph{satisfaction} of timed specifications.

In contrast to these declarative specifications, our TRM formulation gives designers fine-grained control over reward logic, e.g., integrating state-based delay costs and transition-based rewards, allowing standard Markovian rewards on MDPs, etc. Moreover, beyond the digital-clock setting, we also address the real-time interpretation and investigate several heuristics in both cases, theoretically and experimentally.

% In terms of methodology, our work explores the differences related to the different time interpretations (digital vs real-time) and use of heuristics to improve learning, both theoretically and experimentally.}

% %%%%%TRMs can be seamlessly integrate with RL frameworks, allowing for optimisation that exploits the underlying automaton structure.

\paragraph{RL with Non-Markovian Rewards.}
To specify non-Markovian reward specifications, the most widely used formalisms are temporal logics and finite-state machines (FSMs). Among temporal logics, there has been particular focus on Linear Temporal Logic (LTL)~\cite{DBLP:conf/ijcai/CamachoIKVM19,DBLP:conf/cdc/HasanbeigKAKPL19,DBLP:conf/icra/Bozkurt0ZP20,DBLP:conf/nips/JothimuruganBBA21,DBLP:conf/ijcai/ShaoK23}. Most of these approaches focus on translating formulas into FSMs, which are the operational structures used to guide learning.

FSMs are central to non-Markovian RL due to their compositional structure. Reward machines (RMs)~\cite{DBLP:journals/jair/IcarteKVM22} provide a general means of defining rewards in the FSM structure. Reward machines have been extended widely: stochastic transitions~\cite{DBLP:conf/aaai/CorazzaGN22}, $\omega$-regular properties~\cite{DBLP:conf/ecai/HahnPSS0W23}, partial observability~\cite{DBLP:conf/nips/IcarteWKVCM19},  multi-agent settings~\cite{DBLP:conf/atal/Neary00T21} and continuous-time MDPs~\cite{DBLP:conf/aips/FalahG023,DBLP:conf/ijcai/FalahG025}.

None of these works addresses fine-grained timing constraints as considered in this paper.
%Our work builds on this line and, to our knowledge, we are the first to introduce a timed extension of reward machines.

\paragraph{Timed Automata in Control and Planning.}
Timed automata~\cite{DBLP:journals/tcs/AlurD94} are a well-established formalism for modeling and verifying systems with time-dependent behavior~\cite{DBLP:conf/cav/BozgaDMOTY98,DBLP:journals/sttt/LarsenPY97}. Quantitative variants of TA, such as priced timed automata~\cite{DBLP:conf/hybrid/BehrmannFHLPRV01,DBLP:conf/hybrid/BouyerBL04} and weighted timed automata \cite{DBLP:conf/hybrid/AlurTP01}, which are extensions with costs or rewards to states and transitions, have been used in strategy synthesis~\cite{DBLP:conf/cav/BehrmannCDFLL07,DBLP:conf/tacas/DavidJLMT15} and planning~\cite{DBLP:conf/fsttcs/BouyerCFL04,DBLP:journals/fmsd/BouyerBL08,DBLP:conf/icaps/TollundJNTL24}.

Most of these works are model-based, using deductive methods (e.g., symbolic synthesis or game-theoretic techniques) to compute optimal strategies or plans. In contrast, we show that timed reward structures can be operated in a model-free RL setting using statistical methods.

\section{Preliminaries and Background}

\subsection{Markov Decision Process}
A \emph{Markov Decision Process} (MDP)~\cite{DBLP:books/lib/SuttonB2018} is a tuple $\mdp = (S, A, T, R, \gamma)$, where $S$ is a finite set of states, $A$ is a finite set of actions, $T: S \times A \times S \rightarrow [0, 1]$ is the transition function that defines the probability of transitioning from state $s$ to state $s'$ given action $a$, $R: S\times A \times S \rightarrow \mathbb{R}$ is the reward function that defines the immediate reward received after taking action $a$ in state $s$, and $\gamma \in [0, 1)$ is the discount factor.

Often, a labeling function is augmented in the definition of an MDP to capture key features of the system. Formally, a \emph{labeled MDP} is a tuple $\mdp = (S, A, T, R, \gamma, \prop, L)$, where $L: S \times A \rightarrow 2^\prop$ is a labeling function that maps each state-action triplet to a set of propositions $\prop$ that hold true in that context.

A \emph{policy} for an MDP is a mapping $\pi: (S \times A)^* S \rightarrow \Delta(A)$ that defines the distribution over actions for a given history of states and actions. For Markovian rewards, it is sufficient to consider deterministic and positional policies $\pi: S \rightarrow A$.

The expected cumulative reward for a policy $\pi$ is defined as the sum of the immediate rewards received over time, discounted by the factor $\gamma$. For a policy $\pi$, it is defined as:
\begin{equation*}
    V^\pi(s) = \mathbb{E}_{\pi} \left[ \sum_{t=0}^{\infty} \gamma^t R(s_t, a_t) \mid s_0 = s \right]
\end{equation*}
where $s_t$ is the state at time $t$, $a_t$ is the action taken at time $t$.

\subsection{Reinforcement Learning with Q-learning}
Reinforcement learning (RL) learns policies that maximize discounted reward in MDPs. In model-free RL, the agent samples the environment without explicit transition or reward models. Q-learning, a standard model-free method, learns the optimal action-value function $Q: S \times A \rightarrow \mathbb{R}$, the expected return of taking $a$ in $s$ and then following the optimal policy. The Q-learning update rule is given by:
\begin{equation*}
    Q(s, a) \leftarrow Q(s, a) + \alpha \left( R(s, a) + \gamma \max_{a'} Q(s', a') - Q(s, a) \right)
\end{equation*}
where $\alpha$ is the learning rate, $R(s, a)$ is the immediate reward received after taking action $a$ in state $s$, and $s'$ is the next state reached after taking action $a$ in state $s$.

% \subsection{Deep variants of Q-learning}
% % Change for DDPG.

% Deep Q-learning is an extension of Q-learning that uses deep neural networks to approximate the action-value functions. The Deep Q-Network (DQN) algorithm~\cite{DBLP:journals/nature/MnihKSRVBGRFOPB15} uses a neural network $Q^\theta$ to approximate the Q-function, where $\theta$ are the parameters of the neural network. The DQN algorithm updates the parameters $\theta$ of the neural network using gradient descent based on the following loss function:
% \begin{equation}
%     L(\theta) = \mathbb{E}_{(s, a, r, s') \sim D} \left[ \left( r + \gamma \max_{a'} Q^\theta(s', a') - Q^\theta(s, a) \right)^2 \right]
% \end{equation}
% where $D$ is a replay buffer that stores past experiences $(s, a, r, s')$, and the expectation is taken over the samples from the replay buffer.
\section{Problem formulation}
\label{sec:problem}

% We first introduce the \emph{timed reward machines} (TRMs), the non-Markovian reward structure specification that we use to specify the reward structure in our RL problem.

\subsection{Timed Reward Machine (TRM)}
We extend classical timed automata with reward functions for the RL setting. For the reward machine formalism, we follow~\cite{IcarteKVM18}.
In our formalism, reward machines are augmented with a set of \emph{clocks} $X$ which can assume values in the time domain $\mathbb{T}$. We consider two different time settings: (1) a digital-clock setting, where $\mathbb{T} = \mathbb{N}$; and (2) a real-time setting, where $\mathbb{T} = \mathbb{R}_{\geq 0}$.  

A \emph{guard} is a conjunction of constraints of the form $\phi := x \bowtie c$, where $x \in X$ is a clock, $\bowtie\ \in \{<, \leq, =, \geq, >\}$ is a comparison operator, and $c \in \mathbb{N}$ is a non-negative constant. We denote the set of all guards over $X$ by $\Phi(X)$.
Given a set of clocks $X$, propositions $\prop$, a TRM is a finite state machine defined as a tuple $\mathcal{A} = (U, u_0, F, \Delta_u, \Delta_r)$, where:
\begin{itemize}
    \item $U$ is a finite set of states, $u_0 \in U$ is the initial state, and $F\subset U$ is the set of terminal (sink) states;
    \item $\Delta_u \subseteq U \times  2^\prop \times \Phi(X) \times 2^X \times U$ is the transition relation defining the next state given the current state, active propositions in the current state, a guard over clocks, and the clocks to reset. We denote a transition using $\theta$.
    \item $\Delta_r = \Delta^u_r \cup \Delta^\theta_r$, where 
    $\Delta^u_r: U \rightarrow [S \rightarrow \mathbb{R}]$ is the state-based reward function and $\Delta^\theta_r: \Delta_u \rightarrow [S \times A \times S\rightarrow \mathbb{R}]$ is the transition-based reward function.
\end{itemize}
While our TRM formalism is inspired by priced timed automata, it is generalized for RL frameworks, allowing general Markovian reward functions over both states and transitions. In our setting, one can consider negative rewards to be costs.

A timed word $w = (d_0, l_0) (d_1, l_1) \dots (d_{n}, l_{n})\in (\mathbb{T} \times 2^{\prop})^\ast$ is a sequence in which $d_i \in \mathbb{T}$ is a time delay at position $i$ and $l_i \in 2^{\prop}$ is the set of propositions observed after that delay at position $i$. 
A \emph{run} $\mathcal{A}^w$ of a timed reward machine $\mathcal{A}$ on $w$ is a sequence
$(u_0, v_0) \xrightarrow{d_0,\theta_0,r_0} (u_1, v_1) \xrightarrow{d_1, \theta_1, r_1} \ldots \xrightarrow{d_{n},\theta_{n},r_{n}} (u_{n+1}, v_{n+1})$, where, for each $i$, $u_i\in U$ is the state of the TRM, $v_i \in \mathbb{T}^{|X|}$ is the clock valuation, $r_i$ is the reward function, and $\theta_{i}$ is the transition at that position.
The above run satisfies the following conditions:
\begin{itemize}
    \item $u_0$ is the initial state, $v_0(x) = 0$ for all $x \in X$.
    \item let $\theta_{i}= (u_i, l_i, \phi_i, \rho_i, u_{i+1}) \in \Delta_u$, then the clock valuations satisfy the following condition: $v_{i} + d_{i}$ satisfies the guard $\phi_i$, and $v_{i+1} = [\rho_i](v_i + d_{i})$, where $[\rho_i]$ is the reset function that resets the clocks in $\rho_i$ to $0$ and keeps the others unchanged.
    %and for each $x\in X$, $v_i[x] = (v_{i-1}[x]+d_i)\cdot \mathds{1}_{x\notin \rho_i}$, where $\mathds{1}$ is the indicator function;
    \item $r_{i}[0] = \Delta^\theta_r(\theta_{i})$ and $r_{i}[1] = \Delta^u_r(u_{i})$ are the transition- and state-based reward functions at position $i$, resp.
\end{itemize}
  
A TRM $\aut$ is \emph{deterministic}, if for every state $u \in U$, every set of propositions $l \in 2^\prop$, and every pair of guards $\phi_1, \phi_2 \in \Phi(X)$ such that $(u, l, \phi_1, \rho, u') \in \Delta_u$ and $(u, l, \phi_2, \rho', u'') \in \Delta_u$, it holds that $\phi_1 \cap \phi_2 = \emptyset$.
As is a common assumption in the literature on reward machines~\cite{IcarteKVM18}, we only consider deterministic TRMs.

\paragraph{Interpretation of TRM on MDPs.}

We model the underlying environment as an MDP $\mathcal{M}= (S, A', T, \gamma, L, \prop)$\footnote{We here drop the Markovian reward $R$ as it is given by a TRM.}, which extends the standard labeled MDP definition by including timing delays in action space: $A'=\mathbb{T}\times A$. Here, the agent has the option to either act immediately or wait for a chosen amount of time, referred to as \emph{delay}, before taking the next action.
Augmenting the agent with delay actions is analogous to standard settings in timed control and games~\cite{DBLP:conf/fsttcs/BouyerCFL04,DBLP:conf/hybrid/BehrmannFHLPRV01}.

In our setting, the agent’s trajectory takes the form of a sequence $\zeta= s_0\cdot (d_0, a_0)\cdot s_1\cdot (d_1, a_1) \cdots (d_{n},a_{n}) \cdot s_{n+1} \in (S\times (\mathbb{T}\times A))^\ast\times S$, where 
%for every $i \in \{0,\dots,n+1\}$, 
$s_i \in S$ is the MDP state, $d_i\in \mathbb{T}$ is the delay chosen, and $a_i \in A$ is the action taken.
Since RL algorithms typically operate on sampled finite trajectories, we restrict attention to bounded-horizon trajectories in this work.

% A trajectory $\zeta$ induces a timed word $w^\zeta = (d_0+1, L(s_0,a_0,s_1)) \dots$ $(d_n+1, L(s_n,a_n,s_{n+1}))$, which serves as input to the TRM $\mathcal{A}$. The offset of $+1$ in the delays accounts for the assumption that actions are executed after the specified delay.
% This definition aligns with the reward machines based RL settings~\cite{DBLP:phd/ca/Icarte22}; in particular, setting all delays $d_i = 0$ yields the standard untimed word over propositions.

A trajectory $\zeta$ induces a timed word $w^\zeta = (d_0+1, L(s_0,a_0)) \dots$ $(d_n+1, L(s_n,a_n))$, which serves as input to the TRM $\mathcal{A}$.
The $+1$ offset in the delays represents the execution duration of an action in the environment. Our choice of a unit offset aligns better with the definitions of reward machines; setting all delays \(d_i = 0\) yields a standard untimed word. One could, however, easily consider the action duration to be any constant \(c\in\real\) offset.
By doing so, we allow general real-time interpretation of MDPs.

On a timed word $w^\zeta$, the TRM $\mathcal{A}$ produces a run
$\mathcal{A}^\zeta: 
%(u_0, v_0) \xrightarrow{d_0+1,\theta_0,r_0} 
(u_0, v_0)$ $\xrightarrow{d_0+1, \theta_0,r_0} (u_1,v_1) \xrightarrow{d_1+1, \theta_1,r_1} \ldots \xrightarrow{d_{n}+1,\theta_n, r_n} (u_{n+1}, v_{n+1}).$
%, where $r_0 = (\mathbf{0},\mathbf{0})$, $\mathbf{0}$ is the constant zero function. 
We now define the discounted cumulative reward for a trajectory $\zeta$. This definition follows the standard treatment of discounting in decision processes with sojourn times, as presented in~\cite[Eq~11.3.1]{DBLP:books/wi/Puterman94}.

Following that framework, each decision point occurring at time $t_i$ contributes a discounted reward of $\gamma^{t_i} \cdot R_i$ to the total return, where $R_i$ denotes the total reward accumulated during the transition from state $s_i$ to $s_{i+1}$ after taking action $a_i$. The reward $R_i$ consists of two parts: the lump-sum reward $r^\theta_i(s_i,a_i,s_{i+1})$ obtained from the transition $\theta$ in $\mathcal{A}$ and the state-based reward $r^u_i(s_i)$ obtained from the state $u$ in $\mathcal{A}$, accrued over the interval $[t_i,t_{i+1}]$. 
Formally, the total discounted cumulative reward is defined as:
\begin{equation}\label{eq:discounted-cumulative-reward}
    G^\zeta = \sum_{i=0}^{n} \gamma^{t_{i}} \cdot [ r^\theta_i(s_{i},a_i,s_{i+1}) + r^u_i(s_{i})], \text{where}
\end{equation}
\begin{itemize}
\item $t_i = \sum_{j=0}^{i-1} (d_{j}+1)$ for $i\geq 1$, $t_{0}=0$.
\item $r^\theta_i = \Delta^\theta_r(\theta_{i})$ is the transition-based reward at point $i$.
\item $r^u_i = 
\begin{cases} 
    \sum_{t=0}^{d_i-1} \gamma^{t}\Delta^u_r(u_i) = \frac{1-\gamma^{d_i}}{1-\gamma} \Delta^u_r(u_i) & \hspace{-2.2mm}\text{for } \mathbb{T} = \mathbb{N}, \\
    \int_{0}^{d_i} \gamma^{t}\Delta^u_r(u_i) \, {dt} = \frac{1-\gamma^{d_i}}{-\ln(\gamma)} \Delta^u_r(u_i) & \hspace{-2.2mm}\text{for } \mathbb{T} = \mathbb{R}_{\geq 0}
\end{cases}$ \\ is the state-based reward at point $i$.
\end{itemize}
We calculate the state-based reward $r^u_i$ differently for the digital and real-time settings, following standard interpretations~\cite{DBLP:books/wi/Puterman94}. In the digital clock setting, it is accumulated at each timestep during the delay period, whereas in the real-time setting, it is integrated over the delay interval.

\begin{figure}
    \small
        \centering
    % Example 2x2 grid
    \begin{tikzpicture}[scale=1]
        %\fill[red!20] (0,0) rectangle (1,1);

        \fill[green!20] (1,0) rectangle (2,1);
        \fill[green!20] (0,1) rectangle (1,2);
        \fill[red!20] (1,1) rectangle (2,2);

        \foreach \x in {0,1}{
            \foreach \y in {0,1}{
                \draw (\x, \y) rectangle (\x+1, \y+1);
            }
        }
        % State labels at the bottom
        \node at (0+0.17, 0+0.17) {\tiny$s_0$};
        \node at (1+0.17, 0+0.17) {\tiny$s_3$};
        \node at (0+0.17, 1+0.17) {\tiny$s_1$};
        \node at (1+0.17, 1+0.17) {\tiny$s_2$};
        \node at (1-0.21, 1-0.17) {\tiny$-2$};
        \node at (2-0.21, 1-0.17) {\tiny$-1$};
        \node at (1-0.21, 2-0.17) {\tiny$-1$};
        \node at (2-0.21, 2-0.17) {\tiny$-4$};
        \node at (0+0.5, 0+0.5) {\agent};
        \node at (0+0.5, 1+0.5) {$p$};
        \node at (2-0.5, 0+0.5) {$q$};
    \end{tikzpicture}

        \begin{tikzpicture}[->, >=stealth, node distance=4cm, on grid, auto, initial text=]

        \newcommand{\edgelab}[4]{%
    ${#1}$,\ %
    \if\relax\detokenize{#2}\relax $\top$ \else $#2$ \fi,\ %
    \if\relax\detokenize{#3}\relax $\varnothing$ \else $\{#3\}$ \fi,\ %
    #4%
  }
    \node[state, initial, rectangle, rounded corners] (u0) {$u_0,c$};
    \node[state,rectangle, rounded corners] (u1) [right of=u0] {$u_1,c$};
    \node[state,rectangle, rounded corners, fill=gray!10] (u2) [right of=u1] {$u_2,c$};

    % Example transitions with label structure: propositions, guard, reset, reward
    \path
  % u0 -> u1 with theta_1
  (u0) edge
       node[above]{\edgelab{\{p\}}{x>2}{}{5}}
       node[below]{$\theta_1$}
       (u1)

  % u1 -> u2 with theta_3
  (u1) edge
       node[above]{\edgelab{\{q\}}{x>5}{}{10}}
       node[below]{$\theta_3$}
       (u2)

  % loop on u0 with theta_0 and reward r_0
  (u0) edge[loop above]
       node[align=center]
       {$\theta_0$:\edgelab{\{\}}{}{}{0}}
       (u0)

  % loop on u1 with theta_2 and reward r_0
  (u1) edge[loop above]
       node[align=center]
       {$\theta_2$:\edgelab{\{\}}{}{}{0}}
       (u1)
;
\end{tikzpicture}
    \caption{Environment (above) along with TRM objective (below). The cost function $c$ is depicted in the top-right corner of each state.}
    \label{fig:small_example}
\end{figure}

\begin{table*}
\begin{center}\footnotesize
\begin{minipage}[t]{0.6\linewidth}
\centering\textbf{Trajectory $\boldsymbol{\zeta_1}$}\\[-2mm]
\[
\begin{aligned}[t]
&\textrm{Trajectory:}~
&& \zeta_1 = s_0\!\cdot\!(2,\xup{})\!\cdot\! s_1\!\cdot\!(1,\xright{})\, s_2\!\cdot\!(0,\xdown{})\!\cdot\! s_3 \\[2pt]
&\textrm{Timed word:}~
&& w^{\zeta_1}=(2{+}1,\{p\})(1{+}1,\emptyset)(0{+}1,\{q\}) \\[2pt]
&\textrm{TRM run:}~
&& \mathcal{A}^{\zeta_1}=(u_0,[0]) \xrightarrow{{3,\theta_1,(5,-2)}} (u_1,[3]) \xrightarrow{2,\theta_2,(0,-1)} (u_1,[5]) \xrightarrow{1,\theta_3,(10,-4)} (u_2,[6]) \\[2pt]
&\textrm{Digital-time:}~
&& G^{\zeta_1}=\big[5+(-2)(1+\gamma^1)\big]+\gamma^3[0+(-1)]+\gamma^5[10+0]\approx 6.4 \\[2pt]
&\textrm{Real-time:}~
&& G^{\zeta_1}=\big[5+(-2)\!\!\int_{0}^{2}\!\gamma^t dt\big]
+\gamma^3\big[0+(-1)\!\!\int_{0}^{1}\!\gamma^t dt\big]
+\gamma^5[10+0]\approx 6.6
\end{aligned}
\]
\vspace{2mm}
\end{minipage}
\begin{minipage}[t]{0.6\linewidth}
\centering\textbf{Trajectory $\boldsymbol{\zeta_2}$}\\[-2mm]
\[
\begin{aligned}[t]
&\textrm{Trajectory:}~
&& \zeta_2 = s_0\!\cdot\!(2,\xup{})\!\cdot\! s_1\!\cdot\!(0,\xright{})\, s_2\!\cdot\!(1,\xdown{})\!\cdot\! s_3 \\[2pt]
&\textrm{Timed word:}~
&& w^{\zeta_2}=(2{+}1,\{p\})(0{+}1,\emptyset)(1{+}1,\{q\}) \\[2pt]
&\textrm{TRM run:}~
&& \mathcal{A}^{\zeta_2}=(u_0,[0]) \xrightarrow{3,\theta_1,(5,-2)} (u_1,[3]) \xrightarrow{1,\theta_2,(0,-1)} (u_1,[4]) \xrightarrow{2,\theta_3,(10,-4)} (u_2,[6]) \\[2pt]
&\textrm{Digital-time:}~
&& G^{\zeta_2}=\big[5+(-2)(1+\gamma^1)\big]+\gamma^3[0+0]+\gamma^4[10+(-4)]\approx 5.1 \\[2pt]
&\textrm{Real-time:}~
&& G^{\zeta_2}=\big[5+(-2)\!\!\int_{0}^{2}\!\gamma^t dt\big]
+\gamma^3[0+0]
+\gamma^4\big[10+(-4)\!\!\int_{0}^{1}\!\gamma^t dt\big]\approx 5.4
\end{aligned}
\]
\end{minipage}
\end{center}
\caption{Description of trajectories $\zeta_1$ and $\zeta_2$ from Example~\ref{ex:definition} under digital and real-time semantics ($\gamma=0.9$).}
\label{fig:trajectory-example}
\end{table*}

\begin{example}\label{ex:definition}
To explain the definitions, consider the environment and the TRM in Fig.~\ref{fig:small_example}. The environment has two features, $p$ and $q$, which denote moving into states $s_1$ and $s_3$, resp. Starting at $s_0$, the TRM $\mathcal{A}$ requires the agent first to observe $p$ and then $q$, while satisfying simple clock constraints.

We illustrate two trajectories, $\zeta_1$ and $\zeta_2$, for this example in Table~\ref{fig:trajectory-example}. 
The figure also summarizes, for each trajectory, the induced timed words, the TRM runs, and the resulting discounted returns. 
Both trajectories use the same environment actions but differ in their delay choices, which leads to different behavior under the TRM. For instance, $\zeta_1$ waits in a ``good'' state $s_1$, whereas $\zeta_2$ waits in a ``bad'' state $s_2$, incurring a higher cost. 
Consequently, in the digital-clock setting with $\gamma=0.9$, $\zeta_1$ attains a higher discounted return ($G^{\zeta_1}\approx6.4$) than $\zeta_2$ ($G^{\zeta_2}\approx5.1$). 
The exact ordering holds in real-time with the same delays ($\approx6.6$ vs.\ $\approx5.4$).
%, with the values differing due to accumulating rewards over continuous time.
\end{example}

\paragraph{Properties of TRM.} We now make some observations about trajectories and rewards in TRMs.
As in classical timed automata, a trajectory in TRMs can induce arbitrarily large clock values and delays. However, one can bound them to reason about the expected cumulative reward. As standard~\cite{DBLP:journals/tcs/AlurD94}, we rely on the maximum constant $M$ appearing in the guards of the TRM $\mathcal{A}$.
%We define valuations and delays bounded by the constant $M$ as follows:
Formally, 
$\overline{v}[x] = v[x] \text{ if } v[x]\leq M \text{ else } \overline{v}[x] = \infty$ for all $x\in X$, where $\infty$ denotes the clock value exceeds $M$ which follows the usual comparison semantics, e.g., $\infty > 2$ and $\infty \not\leq 3$.
We also define bounded delays as $\overline{d} = d$ if $d < M$ else $\overline{d} = M$.
%we drop the $M$ as it is fixed in the paper.

We extend this definition to trajectories: for a trajectory $\zeta = s_0\cdot (d_0, a_0)\cdots (d_{n},a_{n}) \cdot s_{n+1}$, we define a \emph{delay-bounded} trajectory $\overline{\zeta} = s_0\cdot (\overline{d}_0, a_0)\cdots (\overline{d}_{n},a_{n}) \cdot s_{n+1}$.
We can show that the delay-bounded trajectory induces a ``similar'' run in a TRM $\mathcal{A}$ to the original trajectory $\zeta$. This follows from the fact that any clock value $v[x]$ beyond $M$ has the same behavior with any guards on $x$.

\begin{lemma}~\label{lem:traj_delay_bound}
     Let $\zeta= s_0\cdot (d_0, a_0)\cdots (d_{n},a_{n}) \cdot s_{n+1}$ be a trajectory and $\overline{\zeta} = s_0\cdot (\overline{d}_0, a_0)\cdots (\overline{d}_{n},a_{n}) \cdot s_{n+1}$ be its delay-bounded trajectory. Also, let $\mathcal{A}^\zeta: (u_0, v_0) \xrightarrow{d_0+1,\theta_0,r_0} \ldots \xrightarrow{d_{n}+1,\theta_n, r_n} (u_{n+1}, v_{n+1})$ be the run of $\mathcal{A}$ on $\zeta$. Then, the following holds: $\overline{\zeta}$ has run $\mathcal{A}^{\overline{\zeta}}: (u_0, \overline{v}_0) \xrightarrow{\overline{d}_0,\theta_0,\overline{r}_0} \ldots \xrightarrow{\overline{d}_{n},\theta_n, \overline{r}_n} (u_{n+1}, \overline{v}_{n+1})$ where the $u_i$'s, and $\theta_i$'s remain the same as in $\mathcal{A}^{\zeta}$.
\end{lemma}
\begin{proof}
    The proof follows from the fact that any clock value $v[x]$ beyond $M$ has the same behavior with any guards on $x$. Formally, let us, w.l.o.g., consider $d_i$ to be the first delay in $\zeta$ such that $d_i \geq M$. Let $\phi_i$ be the guard of transition $\theta_i$ and $v_i+d_i+1\models \phi_i$. Then, we have $v_i+\overline{d}_i+1\models \phi_i$ since constants appearing $\phi_i$ are bounded by $M$.
\end{proof}

%While the above lemma shows that the delay-bounded trajectory $\overline{\zeta}$ induces a similar run in the TRM $\mathcal{A}$ as the original trajectory $\zeta$, it does not guarantee that the discounted rewards $G^{\overline{\zeta}}$ and $G^\zeta$ are the same. 
Moreover, under certain reasonable conditions, bounding the delays can improve the discounted reward. These include inducing costs for delaying in states rather than rewarding, and providing a high terminal reward for completing all tasks. Formally, the state–reward function always has negative values, i.e., $\Delta_{r}(u) < 0 \text{ for every } u \in \mathcal{A}$, and we search for ``good'' trajectories $\zeta$ where $G_i^{\zeta} > 0$ for all decision points $i$, $G_i^{\zeta}$ being the discounted return from \(i\) onward. Trajectories as $\zeta$ can occur, e.g., when terminal rewards along $\zeta$ are sufficiently large.
The following lemma presents this idea.
%For these assumptions, we now state the following lemma, which shows that delay-bounded trajectories can achieve a better discounted reward. Intuitively, the result follows from the fact that shorter delays lower state costs and also reduce the discounting applied to terminal rewards.
\begin{restatable}{lemma}{boundedactions}
~\label{lem:bounded-actions}
     Let $\mathcal{A}$ be a TRM, $\zeta$ be a trajectory of $\mathcal{A}$ and $\overline{\zeta}$ be its delay-bounded trajectory.
     If 
     %that corresponding TRM $\mathcal{A}$ and trajectory $\zeta$ satisfy the following conditions: 
     (1) 
     %the state-based rewards are always negative, i.e., 
     $\Delta^u_r(u) < 0$ for all $u \in U$; and (2) %discounted reward 
     $G^\zeta_i>0$ 
     %is positive for every decision time 
      for all $0\leq i\leq n$,
     then $G^{\overline{\zeta}}\geq G^\zeta$.
\end{restatable}
% \proof[Proof Sketch]{
% We show for the digital case here (as it is similar for the real-time case), using backward induction on the length $i$ of the trajectory. As the hypothesis, we consider that $G^{\overline{\zeta}}_i \geq G^\zeta_i$ for every $i>k$. For the decision point $k$, we show $(G^{\overline{\zeta}}_k = R^{\overline{\zeta}}_k + \gamma^{\overline{d}+1} G^{\overline{\zeta}}_{k+1}) \geq (G^\zeta_k = R^\zeta_k + \gamma^{d+1} G^\zeta_{k+1})$. 
% Firstly, the reward $R^{\overline{\zeta}}_k > R^{\zeta}_k$ is higher for trajectory since state rewards accrue less cost: $\frac{1-\gamma^{d}}{1-\gamma}\Delta^u_r(u_k) < \frac{1-\gamma^{\overline{d}}}{1-\gamma}\Delta^u_r(u_k)$. Secondly, the $\gamma^{\overline{d}+1} G^{\overline{\zeta}}_k \geq \gamma^{d+1} G^{\overline{\zeta}}_k \geq \gamma^{d+1} G^{\zeta}_k$.
% \qed
% }
\begin{proof}
We demonstrate the digital case (which is similar to the real-time case) using backward induction on the trajectory length $i$. As the hypothesis, we consider that $G^{\overline{\zeta}}_i \geq G^\zeta_i$ for every $i>k$. For the decision point $k$, we show 
\[(G^{\overline{\zeta}}_k = R^{\overline{\zeta}}_k + \gamma^{\overline{d}+1} G^{\overline{\zeta}}_{k+1}) \geq (G^\zeta_k = R^\zeta_k + \gamma^{d+1} G^\zeta_{k+1}).\] 
Firstly, the reward $R^{\overline{\zeta}}_k > R^{\zeta}_k$ is higher for the trajectory since state rewards accrue less cost: \[\frac{1-\gamma^{d}}{1-\gamma}\Delta^u_r(u_k) < \frac{1-\gamma^{\overline{d}}}{1-\gamma}\Delta^u_r(u_k).\] Secondly, we have: \[\gamma^{\overline{d}+1} G^{\overline{\zeta}}_k \geq \gamma^{d+1} G^{\overline{\zeta}}_k \geq \gamma^{d+1} G^{\zeta}_k\].
\end{proof}

Based on the assumptions of Lem.~\ref{lem:bounded-actions}, we can bound the delay space to $\mathbb{D} = \mathbb{T}\cap [0, M]$. In our setting, policies are defined as $\pi: (S\times \mathbb{D}\times A)^* S \rightarrow \mathbb{D} \times A$, which map a trajectory to a bounded delay and an action. We call such policies \emph{delay-discounted} policies, as the reward functions defined in Equation~\ref{eq:discounted-cumulative-reward} incorporate the discount factor to delays as well.

The expected cumulative reward of a policy $\pi$ is then the expected discounted sum of rewards over all possible trajectories following $\pi$ from state $s$: $V^\pi(s) = \mathbb{E}_{\zeta \sim \pi} [G^\zeta \mid \zeta[0] = s]$.
A policy $\pi$ is \emph{optimal} (resp., $\varepsilon$-optimal) if $V^\pi=V^*$ (resp., $V^\pi \geq V^* - \varepsilon$), where $V^* = \sup_{\pi} V^\pi(s)$.

\begin{problem}
\label{prob:strat-synth}
    Given a TRM $\mathcal{A}$ and a MDP $\mathcal{M}$, find a delay-discounted ($\varepsilon$-) optimal policy $\pi^*$.
    %that maximizes the expected cumulative reward, i.e., $\pi^* = \arg\max_{\pi} V^\pi(s)$.
\end{problem}

In the following sections, we propose algorithms for Problem~\ref{prob:strat-synth}, both in the digital-clock (dc) and real-time (rt) settings.

% The above problem is the main problem which can be tackled in the discrete time setting, $\mathbb{T} = \nat$ or the continuous time setting, $\mathbb{T} = \real_{\geq 0}$.

\section{The Digital Clock Setting}
\label{sec:tabularQ-digital}

\paragraph{Cross-product space.}
The most important aspect of our approach is the construction of a cross-product between the underlying MDP $M$ and the TRM $\mathcal{A}$.

The cross-product MDP $\mathcal{M}^{\otimes} = \mathcal{M}\otimes \mathcal{A}$ is similar to what is done for classical reward machines, except that one needs to keep track of the clock values as well. For this, we again consider the maximum constant $M$ appearing in the guards of clock $x\in X$ of $\mathcal{A}$ and use the symbol $\infty$ for clock values that go beyond $M$.

The cross product $(S^\otimes, A^\otimes, T^\otimes, R^\otimes)$ is defined below, where:
\begin{itemize}
\item $S^\otimes = S \times U \times V$, where $S$ is the set of states of the MDP and $U$ is the set of states of the TRM, and $V = (\{0,\dots,M\}\cup \{\infty\})^{|X|}$ is the set of bounded clock valuations.
\item $A^\otimes = \mathbb{D}\times A$ is the set of actions.
\item $T^\otimes: S^\otimes \times A^\otimes \times S^\otimes \rightarrow [0, 1]$ and  $R^\otimes: S^\otimes \times A^\otimes \times S^\otimes \rightarrow \mathbb{R}$  are the transition function and the reward function, respectively, defined as follows:
\begin{align*}
& T^\otimes((s, u, v), (d, a), (s', u', v')) = T(s, a, s'), \text{and} \\
& R^\otimes((s, u, v), (d, a), (s', u', v')) = r^u(s) + r^\theta(s,a,s'),\\
& \text{ if }\exists \theta = (u,L(s,a),\phi,\rho,u'), \text{s.t.}
\ v+d+1 \models \phi, \text{ and } v' =  [\rho](v + d + 1), \text{ where }
\end{align*}
for all $x\in X, (v + d + 1)[x] \!=\! v[x] + d + 1$ if $v[x] + d + 1 \le M$, otherwise $\infty$; 
$r^u \!=\! \frac{1-\gamma^d}{1-\gamma} \Delta^u_r(u)$, and
$r^\theta \!=\! \Delta^\theta_r(\theta)$.

% $r^u$  and $r^\theta$ \text{are defined as Equation~\ref{eq:discounted-cumulative-reward}}.

% \item $R^\otimes: S^\otimes \times A^\otimes \times S^\otimes \rightarrow \mathbb{R}$ is defined as follows:
% \begin{align*}
%     &R^\otimes((s, u, v), (d, a), (s', u', v')) = r^u(s) + r^\theta(s,a,s'), s.t., \\
%     &\exists \theta = (u,L(s,a,s'),\phi,\rho,u'),\ s.t., 
% \end{align*}
\end{itemize}

\begin{theorem}
    Optimal positional deterministic delay-discounted policy exists for the cross product MDP $\mathcal{M}^\otimes$.
\end{theorem}

The proof holds, following standard results in MDPs~\cite{DBLP:books/wi/Puterman94}, because the cross product $\mathcal{M}^\otimes$ has a finite state and action space.

\paragraph{Q-learning on cross-product space.} We adapt Q-learning to the cross-product space by modifying the Q-value updates as follows:
\begin{align*}
    Q((s,u,v), (d,a)) \leftarrow &Q((s,u,v), (d,a)) + \\
    &\alpha\left( [R + \gamma^{(d+1)} \max_{(d',a')} Q((s',u',v'), (d',a'))] - Q((s,u,v), (d,a)) \right),
\end{align*}
where $R$ is the reward returned by the TRM $\mathcal{A}$ for the transition from $(s,u,v)$ to $(s',u',v')$ on choosing delay of $d$ and action $a$.

The convergence follows from standard results~\cite{DBLP:journals/ml/WatkinsD92}, since $\mathcal{M}^{\otimes}$ is a valid finite MDP with probabilities $p((s',u',v')|(s,u,v),(d,a)) = p(s'|s,a)$.
\begin{theorem}
    Q-learning on the cross-product $\mathcal{M}^{\otimes}$ converges to an optimal policy under standard assumptions: every state $(s,u,v)\in S^\otimes$ and action $(d,a)\in A^\otimes$ is visited infinitely often; and the learning rate $\alpha_t$ is decreased over time such that $\sum_{t=0}^{\infty}\alpha_t = \infty$ and $\sum_{t=0}^{\infty}\alpha^2_t < \infty$.
\end{theorem}

\subsection{Counterfactual Imagining for Delays}
\label{sec:crm-delays}
Adding delay actions substantially enlarges the action space, so we use
\emph{counterfactual imagining} to explore time alternatives efficiently.

During the Q-learning process, given a realised transition \(\langle (s,\bar{u},\bar{v}),(d,a),r,(s',u',v')\rangle\) in the product MDP , we synthesize counterfactual experiences by
varying the TRM states, clock valuations, and delays:
\[
\big\langle (s,u,v),\,(\bar d,a),\,\bar r,\,(s',\bar u',\bar v')\big\rangle,
\]
where $(\bar{u}, \bar{v}) \xrightarrow{\bar{d},\theta,\bar{r}} (\bar{u}', \bar{v}')$ is a single step based on TRM.

Varying valuations \(\bar v\) and delays \(\bar d\) over all possibilities requires adding several alternatives due to potentially large clock range \(\{0,\ldots,M\}\).
To keep the number of counterfactuals manageable, we consider adding only reasonable alternatives for valuations and delays.
First, we only consider valuations $\bar{v}$ that are close to the realized valuation $v$, i.e., \(\|\bar{v} - v\|_\infty \le r_{crm}\) for a fixed radius \(r_{crm}\) (typically less than 5).
Second, we consider delays $\bar{d}$ that enable satisfaction of guards in the alternative TRM state $\bar{u}$ and clock valuation $\bar{v}$.
In particular, we add delays $\bar{d}$ corresponding to all transitions \(\theta= (\bar{u}, L(s,a),\phi,\rho,\bar{u}')\) such that \(\bar v + \bar d+1 \models \phi\).

In contrast, for (untimed) reward machines~\cite{DBLP:phd/ca/Icarte22}, counterfactuals vary only the RM state \(\bar u\), without varying clocks or delays.

% For this, we consider the space of reachable clock valuations using zones \cite{DBLP:conf/rtss/AlurCDHW92}, a standard notion in timed automata. A \emph{zone} is a set of clock valuations satisfying a conjunction of constraints of the form $x \pm c \bowtie y \pm c'$, where $x, y \in X$, $c, c' \in \mathbb{T}$, and $\bowtie \in \{\leq, <, =, >, \geq\}$. In particular, we use the zone abstraction $Z(u)$ in the TRM $\mathcal{A}$, while being in the state $u$.
% For further readings on Zone-based abstractions, we refer to~\cite{DBLP:conf/formats/BouyerGHSS22}.

% where $(\bar{u}, \bar{v}) \xrightarrow{\bar{d},\theta,\bar{r}} (\bar{u}', \bar{v}')$ is a single step run in the TRM $\mathcal{A}$.
% Now, since there are various parameters $u$,$v$ and $d$ that can be used for generating synthetic experiences, what works best needs to be determined empirically. 

\section{The Real-time Clock Setting}
\label{sec:tabularQ-cont}
In the real-time case, clock values and delay actions can be real-valued, and delays can be chosen from the continuous range $[0, M]$, enabling more precise timing of actions. This allows for more possible policies, often leading to better rewards, which we illustrate through the following example.

\begin{example}
Consider the example in Figure~\ref{fig:continuous-example}. In the real-time setting, there exists a positive-valued policy, which can be seen from the trajectory: $\zeta_1 = s_0\cdot (0.1, \xright{})\cdot s_1\cdot (0, \xright{})\cdot s_2$, which achieves a discounted reward of $[5+(-1){(1-\gamma^{0.1})}/{-\ln(\gamma)}]+\gamma^{1.1}[7] \approx 11.13$ for $\gamma = 0.9$. In contrast, there is no positive-valued policy in the digital-clock setting. For instance, similar trajectories in this setting, $\zeta_2 = s_0\cdot (0, \xright{})\cdot s_1\cdot (0, \xright{})\cdot s_2$ and $\zeta_3 = s_0\cdot (1, \xright{})\cdot s_1\cdot (0, \xright{})\cdot s_2$ will achieve discounted rewards of $-10+\gamma^1[7] \approx -3.7$ and $[5+(-1){(1-\gamma^{1})}/{(1-\gamma)}]+\gamma^2[-10] \approx -4.1$, resp.
\end{example}

\begin{figure}[h]
  \centering
\begin{tikzpicture}[->, >=stealth, node distance=5cm, on grid, auto, font=\scriptsize, initial text=]
\newcommand{\edgelab}[4]{%
    $#1$,\ %
    \if\relax\detokenize{#2}\relax $\top$ \else $#2$ \fi,\ %
    \if\relax\detokenize{#3}\relax $\varnothing$ \else $\{#3\}$ \fi,\ %
    #4%
  }
    \node[state, initial, rectangle, rounded corners] (u0) at (-1,0) {$u_0,-1$};
    \node[state,rectangle, rounded corners, fill=gray!10] (u1) at (3,0) {$u_1$};
    \node[state,rectangle] (m1) at (5,0) {\agent};
    \node at (5-0.25, 0-0.25) {\tiny$s_0$};
    \node[state,rectangle] (m2) at (5.95,0) {};
    \node at (5.95-0.25, 0-0.25) {\tiny$s_1$};
    \node[state,rectangle] (m3) at (6.9,0) {$p$};
    \node at (6.9-0.25, 0-0.25) {\tiny$s_2$};

    \path
  (u0) edge [loop above]  node{\edgelab{\{\}}{y>1}{}{5}} (u0)
  (u0) edge [loop below]  node{\edgelab{\{\}}{y\le 1}{}{-10}} (u0)
  (u0) edge [bend left]   node{\edgelab{\{p\}}{x<3}{}{7}} (u1)
  (u0) edge [bend right]  node[below]{\edgelab{\{p\}}{x\ge 3}{}{-10}} (u1)
;
\end{tikzpicture}
\caption{TRM (on the left) and MDP (on the right) illustrating agent behavior in the real-time setting.}
\label{fig:continuous-example}
\end{figure}

Moreover, in contrast to the digital-clock case, in the real-time setting, optimal policies may not exist, as stated below.
\begin{theorem}
    An optimal delay-discounted policy may not exist in the real-time setting.
\end{theorem}
For instance, in Figure~\ref{fig:continuous-example}, the best sequence $(d,\xright{})(0,\xright{})$, $0<d<1$, yields return $G^\zeta = [5 + \frac{(1-\gamma^d)}{-\ln(\gamma)}(-1)] + \gamma^{(1+d)}[7]$. 
This will achieve a supremum of $11.3$ as $d \to 0^+$, but this is unattainable since $d=0$ violates the guard.
% \begin{proof}
%     Consider again the example in Figure~\ref{fig:continuous-example}. In the continuous-time setting, the `best' action sequence that an agent can take is $(d,\xright{})(0,\xright{})$, where $0<d<1$, in order to obtain a positive reward. In this case, the discounted return is $G^\zeta = [5 + \frac{(1-\gamma^d)}{-\ln(\gamma)}(-1)] + \gamma^{(1+d)}[7]$. This attains a supremum of $11.3$ as $d \to 0^+$, but no optimal policy can realize this value since choosing $d=0$ violates the required guard condition.
% \end{proof}

% \begin{tikzpicture}[->, >=stealth, node distance=3.5cm, on grid, auto, font=\scriptsize, initial text=]
%     \node[state, initial, rectangle, rounded corners] (u0) at (0,0) {$u_0,-1$};
%     \node[state,rectangle, rounded corners, accepting] (u1) at (3,0) {$u_1$};
%     \node[state,rectangle] (m1) at (5,0) {\agent};
%     \node[state,rectangle] (m2) at (5.9,0) {$p$};

%     % Example transitions with label structure: propositions, guard, reset, reward
%     \path (u0) edge [bend left] node{$\{p\},x>2,5$} (u1)
%           (u0) edge [bend right] node[below]{$\emptyset,-5$}  (u1);
% \end{tikzpicture}

Hence, our focus will be on learning $\varepsilon$-optimal policies.
However, the usual cross-product MDP $\mathcal{M}^\otimes_{rt}$ has infinite state and action spaces.
Towards this, we consider finite and discrete abstractions of real-time, enabling the use of tabular RL.
Moreover, for the purposes of RL, we restrict our attention to \emph{deterministic positional} policies for the defined abstractions.
% Following the above result, the task in this continuous case is to 

\paragraph{Uniform Discretization.} A naive approach to approximating real-time is to use a uniform discretization.
This would mean partitioning the time domain using a step size $0<\frac{1}{\kappa}<1$, $\kappa>1\in \mathbb{N}$. In this setting, the clock valuations would be the set $V_\kappa = \{v \in ([0,M] \cup \{\infty\})^{|X|} \mid v[x] = c/\kappa \text{ or } v[x]=\infty \text{ for } c = 0,\ldots, M\cdot \kappa, \text{ for all } x \in X\}$ and the delay action space $\mathbb{D}_\kappa = \{c/\kappa \in [0,M] \mid c = 0,\ldots, M\cdot \kappa\}$.

The cross-product MDP $\mathcal{M}^\otimes$ can then be constructed as in the digital-clock case, except that the action space and state space are defined as above. One can therefore apply the Q-learning algorithm developed in Section~\ref{sec:tabularQ-digital} and achieve similar convergence guarantees on the considered cross-product MDP.

To achieve better approximations of the optimal value using uniform discretization, one would require choosing a larger partition $\kappa$. However, the size of the valuation space $|V_\kappa|$ and $|\mathbb{D}_\kappa|$ grows with $\kappa$, specifically $|V_\kappa| = (M\cdot \kappa + 1)^{|X|}$ and $|\mathbb{D}_\kappa| = M\cdot \kappa + 1$. This leads to a significant increase in the state and action space of the cross-product MDP, hindering the scalability of this approach.

\subsection{Corner-Point Abstraction based on Regions}~\label{sec:corner_abstraction}
To address the challenges of discretizing real-time, we adopt principled abstractions of clock values from the timed automata literature. Specifically, we adapt the corner-point abstraction based on region abstraction~\cite{DBLP:conf/rtss/AlurCDHW92}. 
While previously used for priced timed automata~\cite{DBLP:conf/hybrid/BouyerBL04}, we adapt it to an RL setting, interpreting TRMs over MDPs with discounted rewards.
% The intuitive idea behind this corner-point abstraction is that, for a given guard $g$, the best policies would be to choose delays such that the clock-valuations stay close to the \emph{corner-points} defined by $g$. For instance, in Figure~\ref{fig:continuous-example}, the optimal policy is to choose a delay $d$ as close to $0^+$ so that when the action is taken, the guard $y>1$ is just satisfied.

On an intuitive level, regions partition the infinite set of clock values into finitely many equivalence classes that behave identically w.r.t. guards. Region corners, on the other hand, are the integral boundary points of a region. Typically, an RL agent is incentivized to choose delays near region corners to obtain higher returns (as in the example from Figure~\ref{fig:continuous-example}).

To formally introduce the corner-point abstraction, we briefly recall the region abstraction and then define its corner points.
Given a set of clocks $X$ and a max-constant $M$, a \emph{region} is a tuple $(h, [X_0, \ldots, X_p])$, where $h:X\to \{0,\ldots,M\}$, and $(X_i)_{i=0}^p$ is a partition of $X$ such that for all $i>0$, $X_i \neq \emptyset$ and $h(x)=M$ implies $x\in X_0$. A valuation $v$ is in a region if the following conditions hold:
\begin{itemize}
\item for all $x\in X$, $\lfloor v(x) \rfloor = h(x)$,
\item for all $x\in X$, $x\in X_0$ iff $\{v(x)\} = 0$ (i.e., $v(x) = h(x)$), and
\item for all $x,y \in X$, $\{v(x)\} \leq \{v(y)\}$ iff $x\in X_i$, $y\in X_j$, $i\leq j$,
\end{itemize}
where $\lfloor c \rfloor$ and $\{c\}$ denote the integer and fractional parts of $c$, respectively.
For example, the valuation $v$ with $v(x)=1.2$, $v(y)=0.5$ lies in the region $\region=(\{x:1, y:0\}, [\{\},\{x\},\{y\}])$.

Two valuations $v, \overline{v}$ are \emph{region-equivalent}, denoted $v\sim \overline{v}$, if they belong to the same region. For a valuation $v$, $[v]$ denotes the region to which it belongs.
Region-equivalence can be naturally extended to trajectories and policies.

% A corner point is a valuation $v\in \{0,\ldots,M\}^{|X|}$ with integral values for each clock. A corner point of a region $\region$ belongs to the (topological) closure of $\region$. For example, the corner points of the example region above are $(1,0)$, $(1,1)$, and $(2,1)$.

A corner point of a region $\region$ is a valuation $v\in \{0,\ldots,M\}^{|X|}$ with integral values for each clock and belongs to the (topological) closure of $\region$, e.g., the corner points of the example region above are $(1,0)$, $(1,1)$, and $(2,1)$.

We here exploit region corners to search for policies in $\mathcal{M}^\otimes_{rt}$ that yield higher rewards, namely \emph{corner policies}. Intuitively, such policies choose delays so that the clock valuations encountered along each possible trajectory lie at region corners.

To formalize the notion of corner policies, we first introduce some notation.
For a precision $\delta>0$, we define the $\delta$-corners of a region $\region$ to be $\corner_\delta(\region) = \{v \in \region \mid \forall x, v(x) \in (c - \delta, c + \delta) \text{ where } c \text{ is a corner of }\region\}$ the valuations close to its corners. 

Given a trajectory $\zeta = s_0\cdot (d_0, a_0)\cdots (d_{n},a_{n}) \cdot s_{n+1}$, we define its \emph{region-equivalent trajectories} $[\zeta]$ to be the set of all trajectories $\hat{\zeta}$ of the form $s_0\cdot (\hat{d}_0, a_0)\cdots (\hat{d}_{n},a_{n}) \cdot s_{n+1}$ such that for all $i\in\{0,n+1\}$, $v_i\sim \hat{v}_i$,  where $v_i$ and $\hat{v}_i$  are the valuations appearing in $\mathcal{A}^{\zeta}$ and $\mathcal{A}^{\hat{\zeta}}$. 
Given a trajectory $\zeta$ and a precision $\delta$, we define \emph{region-equivalent corner trajectories} of $\zeta$, denoted by $\corner_\delta(\zeta)$, as the set of trajectories $\hat{\zeta}$ such that $\hat{\zeta} \in [\zeta]$ and $\hat{v}_i \in \corner_\delta([v_i])$ for all valuations $v_i$ and $\hat{v}_i$ appearing in the runs $\mathcal{A}^{\zeta}$ and $\mathcal{A}^{\hat{\zeta}}$, respectively.

The following lemma states that for any trajectory $\zeta$, there exists a region-equivalent corner trajectory $\hat{\zeta}$ such that the discounted return $G^{\hat{\zeta}}$ is close to $G^\zeta$ for some high discounting factor $\gamma<1$.
%This happens because $\gamma^{t_i} R_i$ is a strictly decreasing (increasing) function if $R>0$ ($R<0$).
\begin{restatable}{lemma}{trajectorycorner}
\label{lem:corner_trajectory}
    Given any $\varepsilon > 0$ and any trajectory $\zeta$, there exists $\delta>0$ and a region-equivalent corner trajectory $\hat{\zeta} \in \corner_\delta(\zeta)$ such that $G^{\hat{\zeta}} > G^\zeta - \varepsilon$ for some discounting factor $\gamma<1$.
\end{restatable}
\begin{proof}
The proof proceeds, intuitively, by first showing that for any trajectory $\zeta$, the delay-discounted return $G^\zeta$ of the trajectory can be closely approximated by an ``undiscounted'' return $L$ of the trajectory, i.e., $|G^{{\zeta}} - L^\zeta| < \varepsilon/3$ for any precision $\varepsilon > 0$.
Then, we show that for a corner-trajectory $\hat{\zeta}$
$L^{\hat{\zeta}}\geq L^\zeta - \varepsilon/3$, implying that $G^{\hat{\zeta}} > L^{\hat{\zeta}} - \varepsilon/3 \geq L^\zeta - 2\varepsilon/3 > G^\zeta - \varepsilon$.
We formalize this now.

First note that, since each delay is bounded by $M$, each  decision point $t_i \leq nM = T$ for some $T>0$.
For simplicity of calculation of the discounted return, we rewrite the discounting factor $\gamma$ as $e^{-\lambda}$, i.e., $\lambda=-\ln \gamma$, where $\lambda > 0$. 
Therefore, we can write $G_{\lambda}^\zeta$ as follows:
\begin{align}
G_{\lambda}^\zeta
&=
\sum_{i=0}^n
\left[
\gamma^{t_i} \cdot \Delta_i^\theta 
+
\gamma^{t_i} \cdot
\frac{1-\gamma^{d_i}}{-\ln{\gamma}} \cdot \Delta_i^u \right]\\&=
\sum_{i=0}^n
\left[
e^{-\lambda t_i} \cdot \Delta_i^\theta 
+
\frac{e^{-\lambda t_i}-e^{-\lambda (t_{i+1}-1)}}{\lambda} \cdot \Delta_i^u \right]
\\&=
\sum_{i=0}^n
\left[
\bigl(1-\lambda t_i+R_\lambda(t_i)\bigr) \cdot \Delta_i^\theta 
+
\frac{\lambda\bigl((t_{i+1}-1)-t_i\bigr)
+
R_\lambda(t_i)-R_\lambda(t_{i+1}-1)}{\lambda} \cdot \Delta_i^u \right]
\\&=\sum_{i=0}^n
\left[
[
  (\Delta_i^\theta - \Delta_i^u) - \lambda t_i \cdot (\Delta_i^\theta + \Delta_i^u) + t_{i+1} \cdot \Delta_i^u
]
+
[
R_\lambda(t_i) \cdot \Delta_i^\theta
+
\frac{
R_\lambda(t_i)-R_\lambda(t_{i+1}-1)}{\lambda} \cdot \Delta_i^u
] 
\right]
\label{eq:lambdaform}
\end{align}
In step (2), we rewrite using $\gamma = e^{-\lambda}$ and $t_{i+1} = d_i + t_{i} + 1$.
In step (3), we rewrite $e^{-\lambda t}$ as $1-\lambda t+R_\lambda(t)$, using Taylor's expansion, to get the expression involving error terms $R_\lambda(t)$.
In step (4), we simply rearrange the terms.

The above equation shows us that $G_{\lambda}^\zeta$  can be written as $L^{\zeta}_\lambda + E^{\zeta}_\lambda$, where $L^\zeta_\lambda$ is an affine function in $t_i$'s for a fixed $\lambda$ and $E_\lambda^{\zeta}$ is the term that depends on $R_\lambda(t_i)$.

Now, for any $t\leq T$, we can bound the value $R_\lambda(t)$ as: \(R_\lambda(t)=\frac{\lambda^2 t^2}{2}e^{-\lambda \tau}\leq \frac{\lambda^2 T^2}{2} \) for some $0<\tau<t$ when $\lambda > 0$. Therefore, we can also upper-bound the affine function involving $R_\lambda(t_i)$, $|E^{\zeta}_\lambda|<\eta(\lambda)$, using triangle inequality.
Therefore, $|G_{\lambda}^\zeta-L^{\zeta}_\lambda| = |E_\lambda^{\zeta}| \leq \eta(\lambda)$ for any trajectory $\zeta$. We can now simply choose a $\lambda$ such that $\eta(\lambda) < \varepsilon/3$ and fix it.

Now $L^\zeta_\lambda$ is an affine function in $t_i$'s and therefore achieves its supremum at corner points, say whose value is $L^{*}_\lambda$.
Therefore, for a region-corner trajectory
$\hat{\zeta}\in \corner_\delta(\zeta)$, 
we have 
$L^\zeta_\lambda \leq L^{\zeta^*}_\lambda\leq L^{\hat{\zeta}}_\lambda + \delta \cdot K \cdot H,
$
where $H$ is the maximum horizon of the trajectory, and $K\cdot\delta$
bounds the value difference incurred by considering the decision time to be at corners.
One can then choose $\delta$ sufficiently small such that
$\delta \cdot K \cdot H < \varepsilon/3$.
\end{proof}

% \proof[Proof Sketch]{
% The proof proceeds inductively on the length of the trajectory $\zeta$.
% We modify $\zeta=\hat{\zeta}_0$ in each inductive step to $\hat{\zeta}_i$ to ensure that the valuation $\hat{v}_i$ at each decision point $i$ is in  $\corner_\varepsilon([v_i])$.
% In induction step $i$, we assume the hypothesis that in trajectory $\hat{\zeta}_{i-1}$, for all $j < i$, $\hat{v}_j$ are in $\corner_\varepsilon([v_j])$ and $G^{\hat{\zeta}_{i-1}} \geq G^{\hat{\zeta}_{j}}$.
% Now, we create $\hat{\zeta}_{i}$ from $\hat{\zeta}_{i-1}$ by modifying $\hat{d}_i$ to make $\hat{v}_i\in \corner_\varepsilon([v_i])$, also ensuring that the reward $G^{\hat{\zeta}_{i}}\geq G^{\hat{\zeta}_{i-1}}$.
% \qed
% }

We now extend the definition of ``region-equivalent corner sets'' from trajectories to policies.
Let $\pi$ and $\hat{\pi}$ be deterministic positional policies such that, for each state $(s,u,v)$ in the cross product MDP, $\pi(s,u,v) = (d,a)$ and  $\hat{\pi}(s,u,v) = (d',a)$,
i.e., both prescribe the same discrete action $a$ but possibly different delays $d$ and $d'$.
We say that $\hat{\pi} \in \corner_\delta(\pi)$ is a \emph{region-equivalent corner policy} of $\pi$ if for all $(s,u,v)$, $d' = d \text{ if } v \notin \corner_\delta([v])$, and $v + d' \in \corner_\delta([v + d])$ otherwise. 

% Now we can prove the following
We can extend the previous lemma for the following result:

\begin{restatable}{lemma}{policycorner}
\label{lem:corner_policy}
    Given any $\varepsilon>0$ and any delay-discounted policy $\pi$, there exists $\delta$ and a region-equivalent corner policy $\hat{\pi}\in C_\delta(\pi)$ such that  %, trajectory $\hat{\zeta} \in [\zeta]$ such that all valuations $\hat{v}_i \in \corner_\varepsilon([v_i])$, where $v_i$ and $\hat{v}_i$ are the valuations appearing in $\mathcal{A}^{\zeta}$ and $\mathcal{A}^{\hat{\zeta}}$, respectively, and $G^{\hat{\zeta}} \geq G^\zeta$. 
    $V^{\hat{\pi}} > V^\pi - \varepsilon$ for some discounting factor $\gamma < 1$.
\end{restatable}

\begin{proof}
    Since we have fixed a horizon, we can also perform a similar proof of approximating policy $\hat{\pi}$ with $\pi$ as the previous result. Here, the proof proceeds by considering the induced finite-horizon Markov chain $\mdp_\pi$ of the policy. Following a similar calculation, we now rewrite the expected value of the discounted reward $G^{\zeta}_\lambda$ as an affine part and an error term: %a  constant times product for each of the finite trajectories originating from $s$, considering the likelihood from the trajectory from the Markov chain distributiony
\begin{align*}
    \sum_{\zeta\in \mdp_{\pi}} Pr_{\zeta\sim\mdp_\pi} [G^{\zeta}_\lambda] = 
    \sum_{\zeta\in \mdp_{\pi}} Pr_{\zeta\sim\mdp_\pi}([L^{\zeta}_\lambda+E^{\zeta}_\lambda]) 
\end{align*}
Again in the above equation $Pr_{\zeta\sim\mdp_\pi}[L^{\zeta}_\lambda]$ is an affine function, achieving maximum at corner points and $Pr_{\zeta\sim\mdp_\pi}[E^{\zeta}_\lambda]$ is a bounded term, since the horizon is bounded.
\end{proof}

We can therefore focus our attention to only region-equivalent corner policies. Although it significantly reduces the search space, an infinite number of such policies can still exist. To address this, we introduce a `finite' abstraction of the cross-product MDP, where valuations correspond to the corners of the regions. We then show that this finite MDP can provide us a near-optimal approximation of the original real-time MDP.

\paragraph{Cross-product using Corner-Point Abstraction.}
We first describe how a corner configuration $(\region,\alpha)$ of the corner-abstraction of $\mathcal{A}$, where $\region$ is a region and $\alpha$ is its corner point, evolves under elapsing time.
Here, the agent, in addition to a delay $d\in \mathbb{D} = \{0,\ldots,M\}$,  chooses a region successor $\sigma\in\mathbb{S}=\{-2|X|,\ldots,0,\ldots, 2|X|\}$ that assigns which region to move to associated with a corner\footnote{Some successors may be invalid or not distinct for some $(\region,\alpha)$.}.
Intuitively, applying a delay-successor tuple $(d,\sigma)$ to a configuration $(\region,\alpha)$ leads to a new configuration $(\region',\alpha')$ obtained as follows: first shift both $\region$ and $\alpha$ by $d$ time unit, and then choose the $\sigma^{th}$ successor region associated with that corner.
Formally, $(\region,\alpha)\oplus (d,\sigma)$ is the new configuration $(\region',\alpha')$ defined as: $\alpha' = \alpha+d$, $\region''[h]=\region[h]+d$ and $R'$ is the $\sigma^{th}$ successor region of $R''$ associated with $\alpha'$.

\begin{example}\label{ex:region-corner}
Consider a corner configuration $(\region=(\{x:1, y:0\}, [\{x\},\{y\}]), \alpha=(1,0))$. This region contains valuations such as $v(x)=1$, $v(y)=0.1$. Applying delay $(1,0)$ leads to $(\region_1=(\{x:2, y:1\}, [\{x\},\{y\}]), \alpha_1=(2,1))$, which is the same region and corner pair offset by $+1$. This region contains valuations such as $v(x)=2$, $v(y)=1.1$. Alternatively, applying delay $(1,1)$ leads to $(\region_2=(\{x:2, y:1\}, [\{\},\{x\},\{y\}]), \alpha_1=(2,1))$, which is the region successor of $R_1$ associated with the same corner. This region contains valuations such as $v'(x)=2.1$, $v'(y)=1.2$. We illustrate this example in Figure~\ref{ex:region-corner}.
\end{example}

\begin{figure}
\centering
\begin{tikzpicture}[scale=1.4, >=stealth]

  % Parameters
  \def\M{3}

  % Dotted integer grid lines
  \foreach \i in {1,2,3} {
    \draw[dotted, thick, gray!70] (\i,0) -- (\i,\M);
    \draw[dotted, thick, gray!70] (0,\i) -- (\M,\i);
  }

  % Outer square boundary

  % Diagonal region boundaries in each unit square
  \foreach \i in {0,1,2} {
    \foreach \j in {0,1,2} {
      \draw[thick] (\i,\j) -- ({\i+1},{\j+1});
    }
  }

  % Axes with arrows
  \draw[->, thick] (0,0) -- (3.35,0) node[right] {$x$};
  \draw[->, thick] (0,0) -- (0,3.35) node[above] {$y$};

  % Tick labels
  \node[below left] at (0,0) {$0$};
  \foreach \i in {1,2,3} {
    \node[below] at (\i,0) {$\i$};
    \node[left] at (0,\i) {$\i$};
  }
\fill (1,0.1) circle (1.5pt); 
\node at (0.85,0.1) {\scriptsize $a$};
  \node at (1, 0.5) {\scriptsize $R$}; 
  \fill (2,1.1) circle (1.5pt);
  \node at (1.92,1.02) {\scriptsize $b$};
    \node at (2,1.5){\scriptsize $R_1$};
    \fill (2.1,1.2) circle (1.5pt);
     \node at (2.18,1.28) {\scriptsize $c$};
    \node[above right] at (2.1,1.5) {\scriptsize $R_2$};

\draw[decorate,decoration={brace, amplitude=5pt}] (1,0.15) -- (2,1.15) node[midway, xshift=-7pt, yshift=9pt] {\scriptsize $(1,0)$};
\draw[decorate, decoration={brace, amplitude=7pt}] (2.13,1.16) -- (1.05,0.08)  node[midway, xshift=8pt, yshift=-10pt] {\scriptsize $(1,1)$};
\end{tikzpicture}
\caption{Graphical illustration of Example~\ref{ex:region-corner}. Here, corner $a$ represents the configuration $(R,\alpha)$, corner $b$ represents the configuration $(R_1,\alpha_1)$ and corner $c$ represents the configuration $(R_2,\alpha_1)$. To corner $a$, applying a delay successor of $(1,0)$ leads to corner $b$, while applying a delay successor of $(1,1)$ leads to corner $c$.} 
\end{figure}

We then define the cross-product MDP $\mathcal{M}^\otimes = (S^\otimes, A^\otimes, T^\otimes, R^\otimes)$ as follows:
% \begin{itemize}
% \item 
$S^\otimes = S \times U \times \regionset \times \cornerset$, %where $S$ is the set of states of the MDP and $U$ is the set of states of the TRM, 
where $\regionset$ is the set of regions of $\mathcal{A}$, and $\cornerset$ is the set of corner points associated with the regions;
% \item
$A^\otimes = \mathbb{D}\times \mathbb{S}\times A$, where $\mathbb{D} = \{0,1,\ldots,M\}$ %is the set of integer delays, 
and $\mathbb{S} = \{-2|X|,\ldots,0,\ldots,2|X|\}$; %is the set of region successors and $A$ is the set of actions of the MDP.
% \item
and $T^\otimes: S^\otimes \times A^\otimes \times S^\otimes \rightarrow [0, 1]$ and  $R^\otimes: S^\otimes \times A^\otimes \times S^\otimes \rightarrow \mathbb{R}$  are %the transition function and the reward function, respectively, 
defined as follows:
\begin{align*}
& T^\otimes((s, u, R, \alpha), (d, \sigma, a), (s', u', R', \alpha')) \!=\! T(s, a, s'), \text{and} \\
& R^\otimes((s, u, R, \alpha), (d,\sigma, a), (s', u', R', \alpha')) \!=\! r^u(s) \!+\!  r^\theta(s,a,s'), \\ 
& \text{if } \exists \theta = (u,L(s,a),\phi,\rho,u') \text{ and } R'', \text{ s.t. }
R''\models \phi, 
 \text{and } R' =  [\rho](R''),\\ &\text{ where }
 (R'', \alpha'') = (R, \alpha) \oplus (d+1, \sigma), R''\models \phi,
\end{align*}
% where for all $x\in X, (v + d + 1)[x] = v[x] + d + 1$ if $v[x] + d + 1 \le M$, otherwise $\infty$; 
and $r^u = \frac{1-\gamma^d}{-\ln(\gamma)} \Delta^u_r(u)$ and
$r^\theta = \Delta^\theta_r(\theta)$.
% \end{itemize}

% \item $R^\otimes: S^\otimes \times A^\otimes \times S^\otimes \rightarrow \mathbb{R}$ is defined as follows:
% \begin{align*}
%     &R^\otimes((s, u, v), (d, a), (s', u', v')) = r^u(s) + r^\theta(s,a,s'), s.t., \\
%     &\exists \theta = (u,L(s,a,s'),\phi,\rho,u'),\ s.t., r^u \text{ and } r^\theta \text{are defined as Equation~\ref{eq:discounted-cumulative-reward}}
% \end{align*}
The above operations on regions such as successor and reset are well-defined and can be computed efficiently~\cite{DBLP:conf/rtss/AlurCDHW92}.

We now show results that demonstrate that returns from a corner policy in real-time MDP $\mathcal{M}^{\otimes}_{rt}$ can be approximated by a policy in corner-point abstraction MDP  $\mathcal{M}^{\otimes}_{ca}$ and vice-versa.
To this end, we first show the following result:
\begin{lemma}
\label{lem:corner_abstraction}
    For any corner policy $\pi$ in $\mathcal{M}^{\otimes}_{rt}$, there exists $\mu>0$ and a policy $\pi'$ in the corner-point abstraction MDP $\mathcal{M}^{\otimes}_{ca}$ such that $|V^\pi - V^{\pi'}| <\mu$.
\end{lemma}
\begin{proof}
    Fix a corner policy $\pi$ in $\mathcal{M}^{\otimes}_{rt}$.
  We construct a corresponding policy $\pi'$ in the corner-point abstraction
  $\mathcal{M}^{\otimes}_{ca}$ as follows.
  For every cross-product state $(s,u,v)$ such that $v \in \corner_{\delta}([v])$
  and $\pi((s,u,v)) = (d,a)$, define
  \[
    \pi'((s,u,[v],\alpha)) = (d',\sigma,a),
  \]
  where $\alpha$ is the closest corner w.r.t. $v$, and  the parameters $(d',\sigma)$ satisfy:
  \begin{enumerate}
    \item $v + d \in \corner_{\delta}(\region')$ where
          $(\region',\alpha') = ([v],\alpha) \oplus (d',\sigma)$;
    % \item $v$ lies at the same $\varepsilon$-corner of its region $[v]$ as $\alpha$; and
    \item $\alpha' = \alpha + d'$ is the closest corner point of $[v + d]$, i.e., $v+d$ and $v+d'$ 
          correspond to the same corner of their respective regions.
  \end{enumerate}

  By construction, the offsets between the concrete and abstract valuations
  are bounded: $|\alpha - v| \le \delta$, and $|\alpha' - (v+d)| \le \delta$. Hence the induced delay shift satisfies $|d' - d| \le 2\delta$.
  Because rewards in $\mathcal{M}^{\otimes}_{rt}$ are continuous with respect to delay and
  clock valuations inside each region (the guards and resets unchanged),
  the difference in immediate reward between $\pi$ and $\pi'$ at any step is
  bounded by a Lipschitz-continuous function $f(\delta)$ satisfying
  $f(\delta)\to 0$ as $\delta \to 0$.

  The discount factor does not amplify this bound since
  $0 < \gamma^k \le 1$ for all $k$.
  Since we bound the horizon $H$, the total cumulative
  reward discrepancy can be bounded as follows:
  $|V^{\pi} - V^{\pi'}| \le H \cdot f(\delta) = \mu$ 
\end{proof}
%We now have all ingredients to proof~\Cref{thm:eps-optimality}.

We now describe the ``lifting'' process that produces a corner policy $\pi$ in $\mathcal{M}^{\otimes}_{rt}$ from a policy $\pi'$ in $\mathcal{M}^{\otimes}_{ca}$. 
Towards this, fix a policy $\pi'$ and a state in $M^{\otimes}_{ca}$ of the form $(s,u,R,\alpha)$ with $\pi'((s,u,R,\alpha)) = (d',\sigma,a)$ such that $(R',\alpha') = (R, \alpha) \oplus (d',\sigma)$. Then, for all $v \in C_\delta(R)$ with $\alpha$ being the closest corner point of $R$ w.r.t. $v$, we choose $d \in \mathbb{D}$ such that, 
\[
    \pi((s,u,v)) = (d,a),
\]
where $(s,u,v)$ is a state in $\mathcal{M}^{\otimes}_{rt}$ and $(v+d) \in C_{\delta}(R')$ and $\alpha'$ is the closest corner point of $R'$ w.r.t. $v+d$.

The following lemma now shows that the described lifting process produces a corner policy in $\mathcal{M}^{\otimes}_{rt}$ with a value close to the policy in $\mathcal{M}^{\otimes}_{ca}$.
\begin{lemma}\label{lem:lifting}
Given any $\mu>0$ and any policy $\pi'$ in $\mathcal{M}^{\otimes}_{ca}$, the above lifting process returns a corner policy $\pi$ in $\mathcal{M}^{\otimes}_{rt}$ such that $|V^\pi - V^{\pi'}| <\mu$.
\end{lemma}
\begin{proof}
  Fix a state $(s,u,v)$ in $M^{\otimes}_{ca}$ and say that $\pi((s,u,v)) = (d,a)$ which is lifted from $\pi'((s,u,[v],\alpha)) = ((d',\sigma),a)$, where $\alpha$ is the closest corner w.r.t. $v$.

Similar to the proof in Lemma~\ref{lem:corner_abstraction}, we can show that $|d' - d| \le 2\delta$. Again, due to the continuity of rewards within the considered region, we can choose $\delta$ such that:  
 \[
    |V^{\pi} - V^{\pi'}| \le H \cdot f(\delta) < \mu.
    \qedhere
\]
\end{proof}

\begin{theorem}
\label{thm:eps-optimality}
Given any $\varepsilon > 0$, our algorithm via Q-learning  on the corner-point abstraction $\mathcal{M}^{\otimes}_{ca}$  returns an $\varepsilon$-optimal policy $\pi^*$ for the real-time MDP $\mathcal{M}^{\otimes}_{rt}$ for some discounting factor $\gamma< 1$.
\end{theorem}
\begin{proof}
Let $\pi$ be an $\varepsilon$-optimal corner policy for $\mathcal{M}^{\otimes}_{rt}$; the existence of such an $\varepsilon$-optimal corner policy, for some $\delta>0$, is guaranteed by Lemma~\ref{lem:corner_policy} for a sufficiently high discount factor. Moreover, by Lemma~\ref{lem:corner_abstraction}, there exists a policy $\pi'$ for $\mathcal{M}^\otimes_{ca}$ such that $|V^\pi - V^{\pi'}| < \mu$, for some $\mu>0$ depending on $\delta$. Hence, the optimal policy $\pi^*$ for $\mathcal{M}^\otimes_{ca}$ found by Q-learning satisfies $V^{\pi^*} \geq V^{\pi'} > V^\pi - \mu$.
Furthermore, by Lemma~\ref{lem:lifting}, the policy $\pi''$ for $\mathcal{M}^{\otimes}_{rt}$ obtained by lifting $\pi^*$ satisfies
$|V^{\pi''} - V^{\pi^*}| < \mu$.
Therefore, by the triangle inequality, we obtain
$V^{\pi''} > V^\pi - 2\mu$.
Choosing $\delta$ sufficiently small so that $2\mu < \varepsilon$, it follows that $\pi''$, obtained by our algorithm, is an $\varepsilon$-optimal policy for $\mathcal{M}^{\otimes}_{rt}$.
\end{proof}

We also design counterfactual imagining for the corner abstraction, analogous to the digital-time setting (Section~\ref{sec:crm-delays}). By contrast, here we synthesize alternative corner configurations $(\bar{\region},\bar{c})$ within a bounded radius $r_{\mathrm{crm}}$ of the realised configuration $(\region,c)$. Moreover, we add alternative delay and successor actions whenever the resulting configuration $(\bar{\region},\bar{c}) \oplus (\bar{d},\bar{\sigma})$ satisfies the relevant guards.

\section{Evaluation}
\label{sec:experiments}

All the described algorithms were implemented in Python3\footnote{available at \url{https://github.com/ritamraha/Timed-Reward-Machines}} by extending~\cite{toroicarte_reward_machines}. We developed the timing extensions, including region and corner abstractions, for reward machines from scratch.
%The code will be made publicly available upon publication and is now provided in the supplementary material.

To improve learning performance, we employed several heuristics for interpreting TRMs.
First, to reduce clock-valuation space $V$, we assigned clock-specific maximum constants $M_x$ for $x \in X$, a standard optimisation in timed automata.
Second, to reduce delay-space $\mathbb{D}$, we set the maximum delay $M_d$ to the largest constant appearing in guards of the form $x \bowtie c$ with $\bowtie \in \{>, \geq, =\}$. This is not a restriction, as delays larger than $M_d$ only incur additional costs and therefore do not need to be considered for optimal policies.

\subsection{Experimental Results.} 

We address three key research questions here: 
\begin{description}
    \item[RQ1] the performance gains by implementing counterfactual imagining,
    \item[RQ2] the performance difference between the different timing abstractions, and
    \item[RQ3] the scale achieved by performing model-free RL.
\end{description}
As TRMs are a novel contribution to the RL framework, there are no direct baselines for comparison. We discuss the technical difference with other formalisms in Section~\ref{sec:technical-comparison-related-works}.

For our evaluation, we use standard Gym~\cite{towers2024gymnasium} environments: (i) the \emph{Taxi} domain (Figure~\ref{fig:taxi-domain}), with propositions indicating colored pick-up locations and whether the passenger is in the taxi or at the destination; (ii) \emph{Frozen Lake} (Figure~\ref{fig:frozenlake-domain}), augmented with three goals \((a,b,c)\) and ten holes \((h)\), action success probability \(0.8\).
The main TRMs used in the experiments are shown in Figure~\ref{fig:TRMs}. The specifications encoded by these TRMs are described in the captions of these figures. We also provide a detailed description of these TRMs.

%\raj{add details about the scale of experiments}

We use Q-learning with per-episode parameter decay \(0.999\), initial rate \(\alpha_0=0.9\), initial exploration \(\varepsilon_0=0.9\), initial Q-values \(Q_0=10\), $\gamma=0.999$ and maximum global steps of 300~K. For counterfactual imagining (CI), we select the top \(15\) by rewards per transition. We averaged the results of each experiment over 10 independent runs.

% We demonstrate the results using a number of TRMs to highlight performance difference:
% Figure~\ref{fig:intro-taxi} shows the Taxi domain environment, Figure~\ref{TRM:crm_comparison-Taxi},\ref{TRM:disc_vs_cont-Taxi} are the TRMs used in the Taxi domain, and Figure~\ref{TRM:crm_comparison-FrozenLake},\ref{TRM:disc_vs_comp-FrozenLake} show the TRMs used for the frozen lake environment.

\subsubsection*{RQ1: Performance gain for Counterfactual Reasoning}
To analyze the improvement of counterfactual imagining (CI) specific to time, we only choose alternative clock valuations and delays (and not TRM states) for both the digital and the real-time settings. We demonstrate this comparison on the Taxi domain with TRM1 (Figure~\ref{fig:trm1}), and on the frozen lake with TRM2 (Figure~\ref{fig:trm2}). TRM1 requires the Taxi agent to pick up a passenger, visit a green location, and drop them at the destination, while satisfying several timing constraints. On the other hand, TRM2 requires the Frozen Lake agent to satisfy three objectives, $a, b, c$, sequentially, while avoiding falling into the holes, and it must also move slowly.

Figure~\ref{fig:plot-crm_comparison} compares discounted returns and episode time (including delays) during Q-learning on both environments. Counterfactuals yield significantly higher returns in both digital and corner-point abstractions by enabling exploration of additional ways to satisfy timing constraints. They also significantly reduce episode time, allowing agents to complete tasks faster.

\subsection*{RQ2: Comparison of Timing Abstractions}
We evaluate the performance of different cross-products for different time interpretations: (i) digital clock abstraction, (ii) uniform discretization with $1/\kappa \in \{0.2,0.5\}$, (iii) corner-point abstraction, and (iv) reward machines. Note that the reward machine interpretation cannot choose delay actions, as it is not designed for timed specifications.
We demonstrate this comparison on the Taxi domain with TRM3 (Figure~\ref{fig:trm3}), and on the frozen lake with TRM4 (Figure~\ref{fig:trm4}).
The TRMs are similar to the previous experiment, with different time constraints.
We provide another TRM in Figure~\ref{TRM:cont-is-better}, which shows a similar but more pronounced comparison between digital clock abstraction and corner-point abstraction in Figure~\ref{plot:cont-is-better}.

\begin{figure}[t]
    \centering
    \begin{subfigure}[b]{0.48\textwidth}
        \centering
        \includegraphics[width=\linewidth]{data/taxi-domain.png}
        \caption{Taxi domain}
        \label{fig:taxi-domain}
    \end{subfigure}
    \hfill
    \begin{subfigure}[b]{0.48\textwidth}
        \centering
        \includegraphics[width=0.8\linewidth]{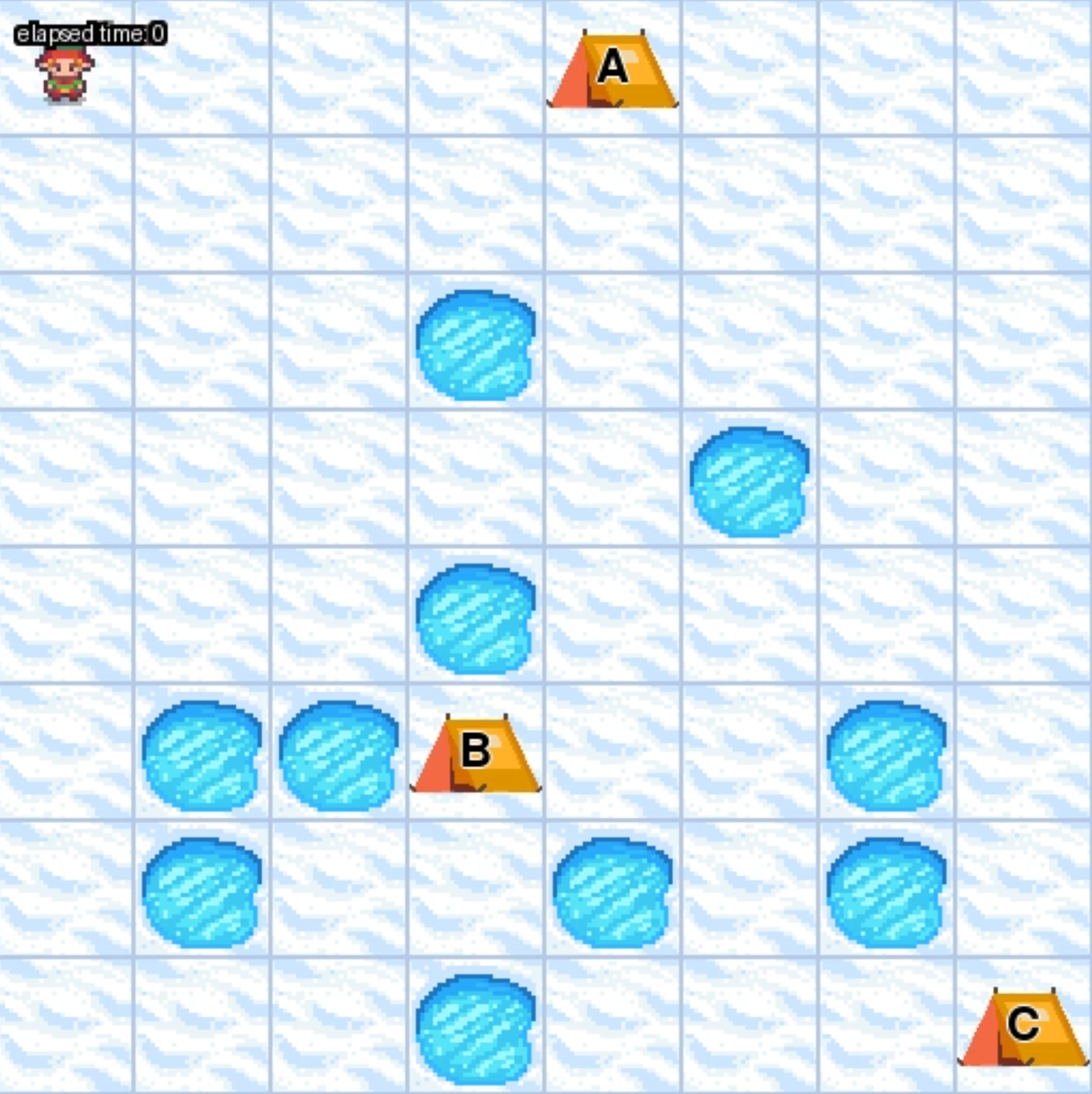}
        \caption{Frozen Lake}
        \label{fig:frozenlake-domain}
    \end{subfigure}
    \caption{Gym environments used in experiments.}
    \label{fig:gym-environments}
\end{figure}

\begin{figure*}[ht]
\begin{subfigure}[b]{0.5\linewidth}
  \centering
    \begin{tikzpicture}
  \node[anchor=south west, inner sep=0] (A) at (-1,0) {\includegraphics[width=.5\linewidth]{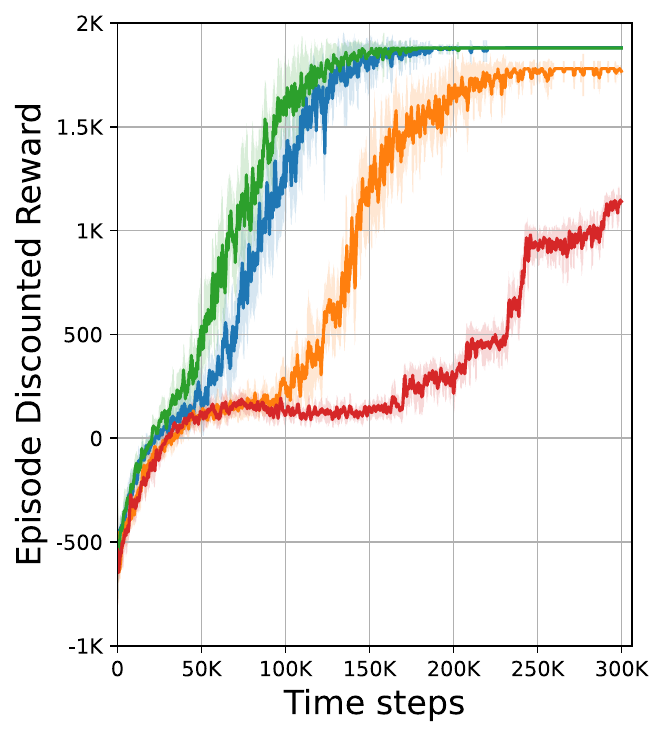}};
  \node[anchor=south west, inner sep=0] (B) at (3,0){\includegraphics[width=.5\linewidth]{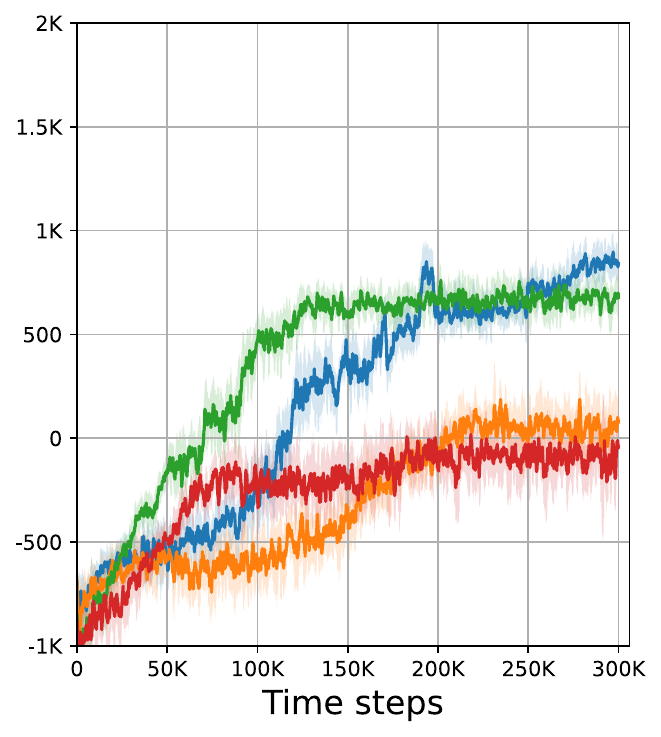}};
  \node[inner sep=0] (C) at (7.1,5){\includegraphics[width=1.7\linewidth]{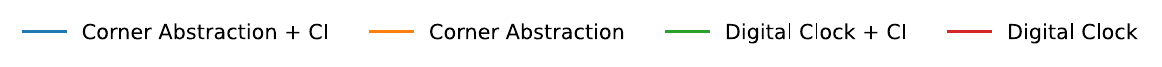}};
  \node[xshift=2.5mm] at ($(A.north)$) {\scriptsize Taxi Domain on TRM1};
  \node[xshift=1mm] at ($(B.north)$) {\scriptsize Frozen Lake on TRM2};
\end{tikzpicture}
\caption{Discounted reward comparison}
\label{fig:plot-crm_comparison-environments}
\end{subfigure}
\begin{subfigure}[b]{0.5\linewidth}
  \centering
    \begin{tikzpicture}
  \node[anchor=south west, inner sep=0] (B) at (3.825,0){\includegraphics[width=.5\linewidth]{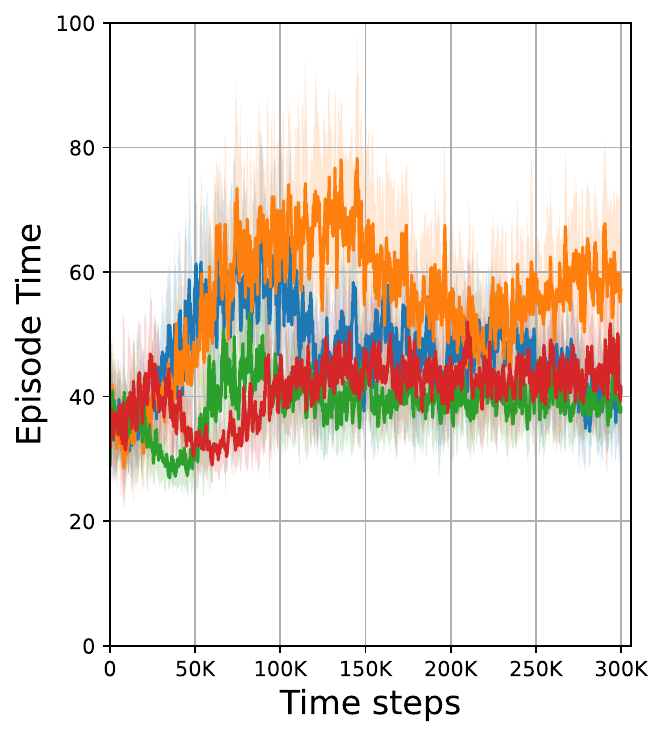}};
  \node[anchor=south west, inner sep=0] (A) at (0,0) {\includegraphics[width=.5\linewidth]{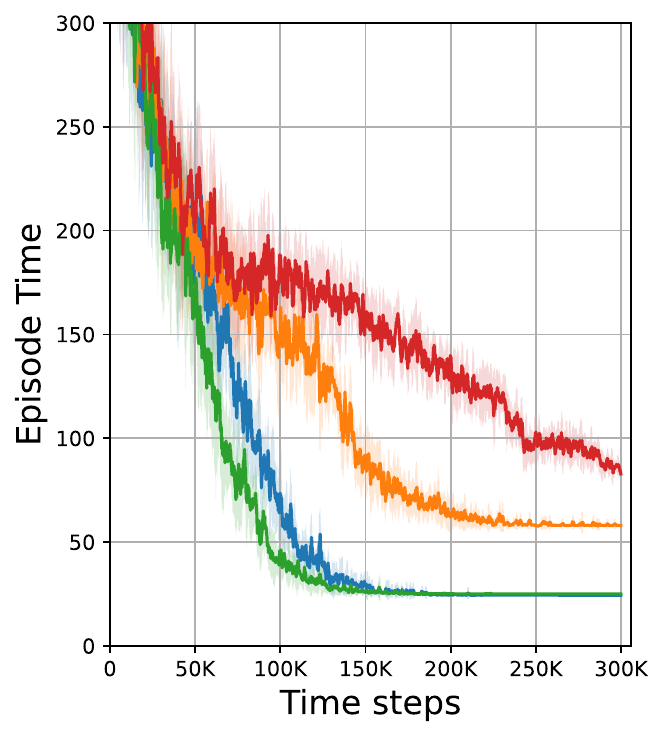}};
  \node[xshift=2.5mm] at ($(A.north)$) {\scriptsize Taxi Domain on TRM1};
  \node[xshift=1mm] at ($(B.north)$) {\scriptsize Frozen Lake on TRM2};
\end{tikzpicture}
\caption{Episode time comparison}
\label{fig:plot-crm_comparison-time}
\end{subfigure}
\caption{RQ1: Performance gain for counterfactual imagining for digital and real-time settings for two environments.}
\label{fig:plot-crm_comparison}
\end{figure*}

\begin{figure*}
\begin{subfigure}[b]{0.5\linewidth}
    \begin{tikzpicture}
  \node[anchor=south west, inner sep=0] (A) at (-1,0) {\includegraphics[width=.5\linewidth]{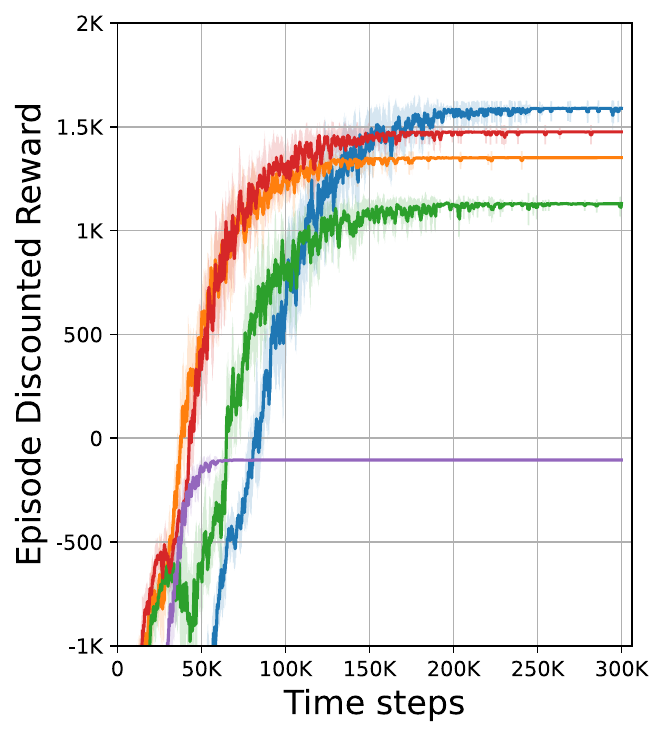}};
  \node[anchor=south west, inner sep=0] (B) at (3,0){\includegraphics[width=.5\linewidth]{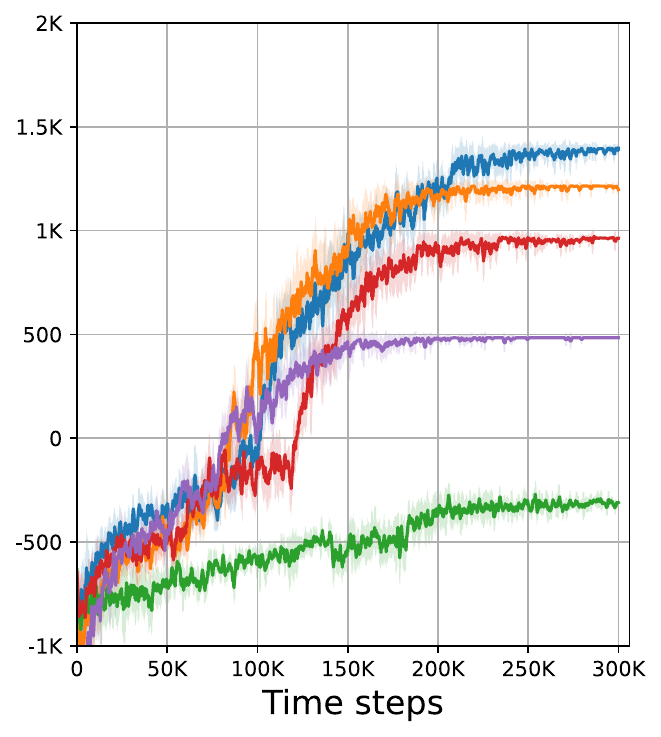}};
 \node[inner sep=0] (C) at (7.4,5){\includegraphics[width=2.1\linewidth]{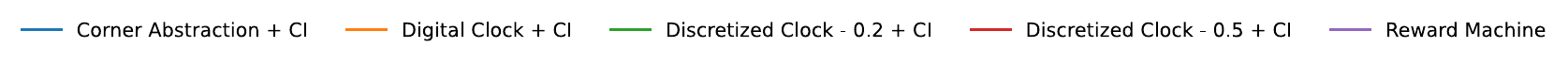}};
  \node[xshift=2.5mm] at ($(A.north)$) {\scriptsize Taxi Domain on TRM3};
  \node[xshift=1mm] at ($(B.north)$) {\scriptsize Frozen Lake on TRM4};
\end{tikzpicture}
\caption{Discounted reward comparison}
\label{fig:plot-abstraction_comparison-environments}
\end{subfigure}
\begin{subfigure}[b]{0.5\linewidth}
    \begin{tikzpicture}
  \node[anchor=south west, inner sep=0] (A) at (0,0) {\includegraphics[width=.5\linewidth]{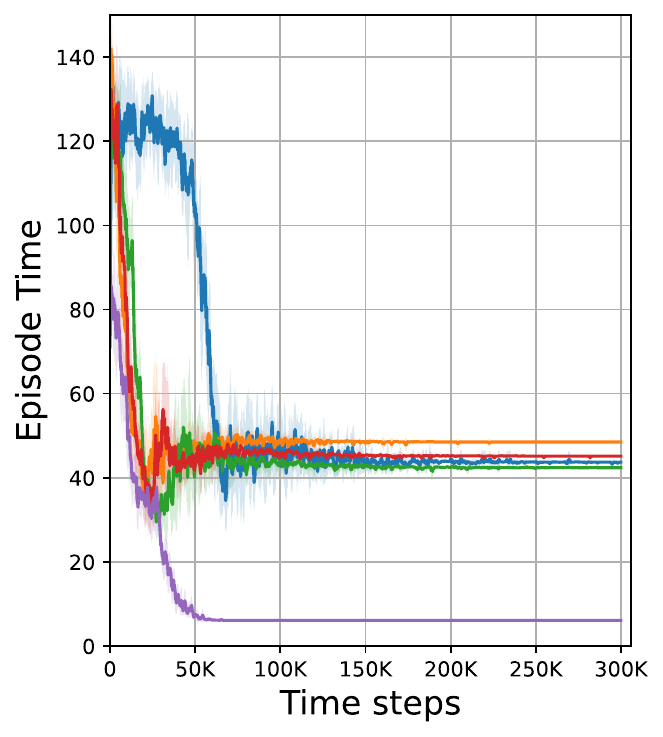}};
  \node[anchor=south west, inner sep=0] (B) at (4,0){\includegraphics[width=.5\linewidth]{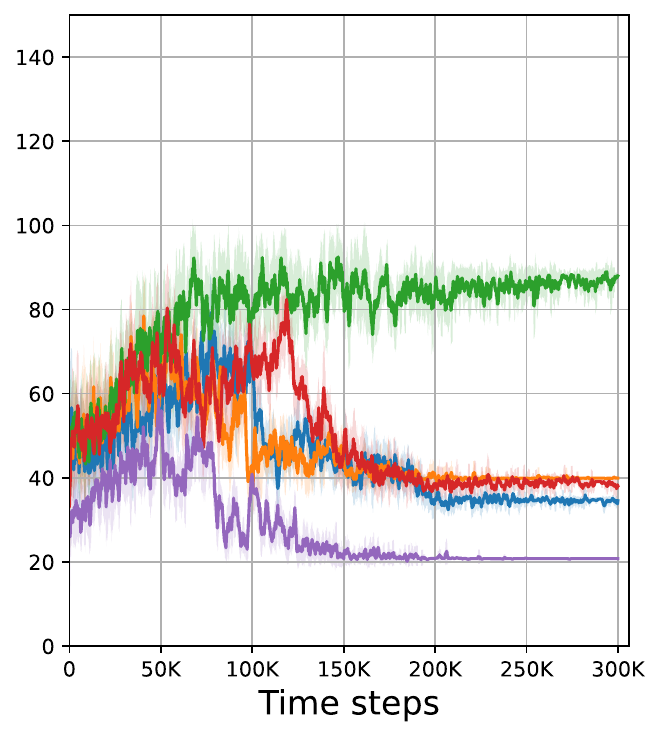}};
  \node[xshift=2.5mm] at ($(A.north)$) {\scriptsize Taxi Domain on TRM3};
  \node[xshift=1mm] at ($(B.north)$) {\scriptsize Frozen Lake on TRM4};
\end{tikzpicture}
\caption{Episode time comparison}
\label{fig:plot-abstraction_comparison-time}
\end{subfigure}
\caption{RQ2: Performance difference for various timed interpretations}
\label{fig:plot-abstraction_comparison}
\end{figure*}

\begin{figure*}[h]
\centering
\begin{subfigure}{\linewidth}
\centering

% \centering
% \resizebox{\columnwidth}{!}{%
\begin{tikzpicture}[
  ->, >=stealth, on grid, auto, font=\footnotesize,
  node distance=4.2cm,
  state/.style={rectangle, rounded corners, draw, minimum width=12mm, minimum height=6mm},
  sink/.style={state, fill=gray!10},
  lab/.style={inner sep=1pt, align=center},
  every loop/.style={looseness=6},
  shorten >=1pt, shorten <=1pt
]
  % Label helpers (horizontal labels; show ⊤ when guard empty, ∅ when reset empty)
  \newcommand{\stlab}[2]{$u_{#1}$,\,#2}
  \newcommand{\edgelab}[4]{%
    \tiny
    \texttt{#1},\ %
    \if\relax\detokenize{#2}\relax $\top$ \else $#2$ \fi,\ %
    \if\relax\detokenize{#3}\relax $\varnothing$ \else $\{#3\}$ \fi,\ %
    #4%
  }

  % Nodes (u1 initial; u4 accepting; u0 sink)
  \node[state, initial, initial text=] (u1) at (0,0)    {\stlab{1}{-1}};
  \node[state]          (u2) at (4.2,0)  {\stlab{2}{-1}};
  \node[state]          (u3) at (9,0)  {\stlab{3}{-1}};
  \node[state]         (u4) at (9,-1.6) {\stlab{4}{-1}};
  \node[sink]           (u0) at (4.2,-1.6) {\stlab{0}{-1}};

  % --- Transitions from u1 ---
  \path
    (u1) edge node[lab, above, pos=0.48]{\edgelab{in\_taxi}{x>10}{x}{200}} (u2)
    (u1) edge[loop above] node[lab]{\edgelab{!in\_taxi}{}{}{-5}} (u1);

  % --- Transitions from u2 ---
  \path
    (u2) edge node[lab, above, pos=0.52]{\edgelab{at\_green\&in\_taxi}{}{x}{400}} (u3)
    (u2) edge[loop above] node[lab]{\edgelab{!at\_green\&in\_taxi}{}{}{-5}} (u2)
    (u2) edge node[lab, left, pos=0.5]{\edgelab{!in\_taxi}{}{}{-5}} (u0);

  % --- Transitions from u3 ---
  \path
    (u3) edge node[lab, right, pos=0.48,align=center]{\tiny\texttt{at\_dest},\\[-1mm] \tiny$x\le 15$,\\[-1mm] \tiny$\varnothing$, \tiny600} (u4)
    (u3) edge[loop above] node[lab]{\edgelab{!at\_dest\&in\_taxi}{}{}{-5}} (u3)
    (u3) edge node[lab, above, pos=0.55, sloped]{\edgelab{!at\_dest\&!in\_taxi}{}{}{-5}} (u0);

  % --- Transitions from u4 ---
  \path
    (u4) edge node[lab, align=center]
  { \edgelab{drop\_off}{}{}{800} \\[-1mm] \edgelab{!drop\_off}{}{}{-5} }
(u0);
\end{tikzpicture}
\caption{TRM1 for Taxi in RQ1}\label{fig:trm1}
\vspace{4mm}
\end{subfigure}

\begin{subfigure}{\linewidth}
% \centering
\centering
% \caption{TRM1 for comparing counterfactual imagining in the Taxi domain}
% \label{TRM:crm_comparison-Taxi}
% \end{figure*}

% \begin{figure*}
% \centering
% \resizebox{\columnwidth}{!}{%
  % ->, >=stealth, on grid, auto, font=\footnotesize,
  % every node/.style={transform shape},
  % line width=0.4pt,
  % node distance=3.2cm,
  % state/.style={rectangle,rounded corners,draw,minimum width=9mm,minimum height=4.5mm,inner sep=1pt},
  % sink/.style={state,fill=gray!10},
  % lab/.style={inner sep=0.5pt},
  % shorten >=0.5pt, shorten <=0.5pt
{
\begin{tikzpicture}[
  ->, >=stealth, on grid, auto, font=\footnotesize,
  every node/.style={transform shape},
  node distance=4.2cm,
  state/.style={rectangle, rounded corners, draw, minimum width=12mm, minimum height=4.5mm},
  sink/.style={state, fill=gray!10},
  lab/.style={inner sep=1pt, align=center},
  every loop/.style={looseness=6},
  shorten >=1pt, shorten <=1pt
]
  % Helpers: show ⊤ for empty guard, ∅ for empty reset; allow multi-resets like {x,y}
  \newcommand{\stlab}[2]{$u_{#1}$,\,#2}
  \newcommand{\edgelab}[4]{%
    \tiny
    \texttt{#1},\ %
    \if\relax\detokenize{#2}\relax $\top$ \else $#2$ \fi,\ %
    \if\relax\detokenize{#3}\relax $\varnothing$ \else $\{#3\}$ \fi,\ %
    #4%
  }

  % Nodes: u1 (init), u2, u3 on a line; sink u0 below u3
  \node[state, initial, initial text=] (u1) at (0,0)       {\stlab{1}{-20}};
  \node[state]          (u2) at (4.2,0)     {\stlab{2}{-20}};
  \node[state]          (u3) at (8.4,0)     {\stlab{3}{-20}};
  \node[sink]           (u0) at (4.2,-1.6)  {\stlab{0}{-200}};

  % --- From u1 ---
  \path
    (u1) edge node[lab, above, pos=0.5]{\edgelab{a}{x\le 12}{x,y}{200}} (u2)
    (u1) edge[loop below] node[lab]{\edgelab{b $\mid$ c}{}{}{-10}} (u1)
    (u1) edge[loop above] node[lab]{\edgelab{\{\}}{y>1}{y}{-5}\\[-1mm]\edgelab{\{\}}{y\le 1}{y}{-50}} (u1)
    (u1) edge node[lab, above,sloped]{\edgelab{h}{}{}{-200}} (u0);

  % --- From u2 ---
  \path
    (u2) edge node[lab, above, pos=0.5]{\edgelab{b}{x\le 15}{x,y}{600}} (u3)
    % (u2) edge[loop below] node[lab]{~|~} (u2)
    (u2) edge[loop above] node[lab]{\edgelab{\{\}}{y>1}{y}{-5}\\[-1mm]\edgelab{\{\}}{y\le 1}{y}{-50}\\[-1mm]\edgelab{a $\mid$ c}{}{}{-10}} (u2)
    (u2) edge node[lab, right]{\edgelab{h}{}{}{-200}} (u0);

  % --- From u3 ---
  \path
    (u3) edge node[lab, below,pos = 0.4,yshift =-11pt]{\edgelab{c}{x\le 10}{}{800}\\[-1mm]\edgelab{h}{}{}{-200}} (u0)
    (u3) edge[loop below] node[lab]{\edgelab{a}{}{}{-10}\\[-1mm]\edgelab{b}{}{}{-20}} (u3)
    (u3) edge[loop above] node[lab]{\edgelab{\{\}}{y>1}{y}{-5}\\[-1mm]\edgelab{\{\}}{y\le 1}{y}{-50}} (u3);
\end{tikzpicture}}
\caption{TRM2 for FrozenLake in RQ1}\label{fig:trm2}
\vspace{4mm}
\end{subfigure}
% \vspace{-2mm}
% \caption{TRMs for comparing counterfactual imagining}
% \label{fig:TRM:crm_comparison}
% \end{figure*}

% \begin{figure*}
% \centering
\begin{subfigure}{\linewidth}
\centering
\begin{tikzpicture}[
  ->, >=stealth, on grid, auto, font=\footnotesize,
  every node/.style={transform shape},
  node distance=4.2cm,
  state/.style={rectangle, rounded corners, draw, minimum width=12mm, minimum height=4.5mm},
  sink/.style={state, fill=gray!10},
  lab/.style={inner sep=1pt, align=center},
  every loop/.style={looseness=6},
  shorten >=1pt, shorten <=1pt
]
  % Label helpers (horizontal labels; show ⊤ when guard empty, ∅ when reset empty)
  \newcommand{\stlab}[2]{$u_{#1}$,\,#2}
  \newcommand{\edgelab}[4]{%
    \tiny
    \texttt{#1},\ %
    \if\relax\detokenize{#2}\relax $\top$ \else $#2$ \fi,\ %
    \if\relax\detokenize{#3}\relax $\varnothing$ \else $\{#3\}$ \fi,\ %
    #4%
  }

  % Nodes (u1 initial; u4 accepting; u0 sink)
  \node[state, initial, initial text=] (u1) at (0,0)    {\stlab{1}{-20}};
  \node[state]          (u2) at (4.8,0)  {\stlab{2}{-20}};
  \node[state]          (u3) at (9.8,0)  {\stlab{3}{-20}};
  \node[state]         (u4) at (9.8,-1.6) {\stlab{4}{-20}};
  \node[sink]           (u0) at (4.8,-1.6) {\stlab{0}{-20}};

  % --- Transitions from u1 ---
  \path
    (u1) edge node[lab, above, pos=0.48]{\edgelab{in\_taxi}{x\le 14}{x,y}{200}} (u2)
    (u1) edge[loop above] node[lab,align=center]{\edgelab{!in\_taxi}{y>1}{y}{-5}\\[-1mm]\edgelab{!in\_taxi}{y\le 1}{y}{-50}} (u1);

  % --- Transitions from u2 ---
  \path
    (u2) edge node[lab, above, pos=0.52]{\edgelab{at\_green\&in\_taxi}{}{x,y}{400}} (u3)
    (u2) edge[loop above,align=center] node[lab]{\edgelab{!at\_green\&in\_taxi}{y>1}{y}{-5}\\[-1mm] \edgelab{!at\_green\&in\_taxi}{y\le 1}{y}{-50}} (u2)
    (u2) edge node[lab, left]{\edgelab{!in\_taxi}{}{}{ -100}} (u0);

  % --- Transitions from u3 ---
  \path
    (u3) edge node[lab, right, pos=0.48,align=center]{\tiny\texttt{at\_dest},\\[-1mm] \tiny \texttt{$x\le 15$},\\[-1mm] \tiny$\varnothing$, \tiny 600} (u4)
    (u3) edge[loop above] node[lab,align=center]{\edgelab{!at\_dest\&in\_taxi}{y>1}{}{-5}\\[-1mm]\edgelab{!at\_dest\&in\_taxi}{y\le 1}{}{-50}} (u3)
    (u3) edge node[lab, above, sloped,]{\edgelab{!at\_dest\&!in\_taxi}{}{}{-100}} (u0);

  % --- Transitions from u4 ---
  \path
    (u4) edge node[lab, align=center]
  { \edgelab{drop\_off}{}{}{800} \\[-1mm] \edgelab{!drop\_off}{}{}{-5} }(u0);
\end{tikzpicture}
\caption{TRM3 for Taxi in RQ2}\label{fig:trm3}
\vspace{4mm}
\end{subfigure}

\begin{subfigure}{\linewidth}
\centering
\begin{tikzpicture}[
  ->, >=stealth, on grid, auto, font=\footnotesize,
  node distance=4.2cm,
  state/.style={rectangle, rounded corners, draw, minimum width=12mm, minimum height=6mm},
  sink/.style={state, fill=gray!10},
  lab/.style={inner sep=1pt, align=center},
  every loop/.style={looseness=6},
  shorten >=1pt, shorten <=1pt
]
  % Helpers: show ⊤ for empty guard, ∅ for empty reset
  \newcommand{\stlab}[2]{$u_{#1}$,\,#2}
  \newcommand{\edgelab}[4]{%
  \tiny
    \texttt{#1},\ %
    \if\relax\detokenize{#2}\relax $\top$ \else $#2$ \fi,\ %
    \if\relax\detokenize{#3}\relax $\varnothing$ \else $\{#3\}$ \fi,\ %
    #4%
  }

  % Nodes (u1 initial; u0 is sink)
  \node[state, initial, initial text=] (u1) at (0,0)      {\stlab{1}{-20}};
  \node[state]          (u2) at (4.5,0)    {\stlab{2}{-20}};
  \node[state]          (u3) at (9,0)    {\stlab{3}{-20}};
  \node[sink]           (u0) at (4.5,-1.6) {\stlab{0}{-200}};

  % --- From u1 ---
  \path
    (u1) edge node[lab, above, pos=0.5]{\edgelab{a}{}{x}{200}} (u2)
    % combine b | c
    (u1) edge[loop below] node[lab, align=center]
         {\edgelab{b $\mid$ c}{}{}{-10}} (u1)
    % default via x with wrapping
    (u1) edge[loop above] node[lab, align=center]
         {\edgelab{\{\}}{x>1}{x}{-5} \\[-1mm] \edgelab{\{\}}{x\le 1}{x}{-50}} (u1)
    (u1) edge node[lab, above,sloped]{\edgelab{h}{}{}{-200}} (u0);

  % --- From u2 ---
  \path
    (u2) edge node[lab, above, pos=0.5]{\edgelab{b}{}{x}{600}} (u3)
    (u2) edge[loop above] node[lab, align=center]
     {\edgelab{\{\}}{x>0}{x}{-5} \\[-1mm] \edgelab{\{\}}{x\le 0}{x}{-50}\\[-1mm] \edgelab{a $\mid$ c}{}{}{-10}} (u2)
    (u2) edge node[lab, right]{\edgelab{h}{}{}{-200}} (u0);

  % --- From u3 ---
  \path
    (u3) edge node[lab, below,pos = 0.4,yshift =-11pt]
         {\edgelab{c}{}{}{800} \\[-1mm] \edgelab{h}{}{}{-200}} (u0)
    (u3) edge[loop below] node[lab, align=center]
         {\edgelab{a}{}{}{-10} \\[-1mm] \edgelab{b}{}{}{-20}} (u3)
    (u3) edge[loop above] node[lab, align=center]
         {\edgelab{\{\}}{x>1}{x}{-5} \\[-1mm] \edgelab{\{\}}{x\le 1}{x}{-50}} (u3);
\end{tikzpicture}
\caption{TRM4 for FrozenLake in RQ2}\label{fig:trm4}
\end{subfigure}
\caption{TRMs for Experimental Evaluations}
\label{fig:TRMs}
\end{figure*}
\FloatBarrier

\begin{figure}[h]
\centering
{%
\begin{tikzpicture}[
  ->, >=stealth, on grid, font=\scriptsize,
  node distance=7cm,
  state/.style={rectangle, rounded corners, draw, minimum width=12mm, minimum height=6mm},
  sink/.style={state, fill=gray!10},
  lab/.style={inner sep=1pt},
  every loop/.style={looseness=7, min distance=8mm},
  shorten >=1pt, shorten <=1pt
]
  \newcommand{\stlab}[2]{$u_{#1}$,\,#2}
  \newcommand{\edgelab}[4]{%
    \texttt{#1},\ %
    \if\relax\detokenize{#2}\relax $\top$ \else $#2$ \fi,\ %
    \if\relax\detokenize{#3}\relax $\varnothing$ \else $\{#3\}$ \fi,\ %
    #4%
  }

  % Nodes: u1 (init), u2, u3 on a line; sink u0 below between u2 & u3
  \node[state, initial, initial text=] (u1) at (0,0)      {\stlab{1}{-20}};
  \node[state]                          (u2) at (4.3,0)    {\stlab{2}{-20}};
  \node[state]                          (u3) at (0,-2.4)    {\stlab{3}{-20}};
  \node[sink]                           (u0) at (4.3,-2.4) {$u_0$};

  % --- From u1 ---
  \path
    (u1) edge node[lab, above, pos=0.52]{\edgelab{a}{x\le 5}{x,y}{1000}} (u2)
    (u1) edge node[lab, left, pos=0.5]{\edgelab{b}{}{x,y}{800}} (u3)
    (u1) edge node[lab, right, pos=0.5]{\edgelab{h}{}{}{-800}} (u0)
    (u1) edge[loop above] node[lab]{\edgelab{!a\&!b\&!h}{y>1}{y}{-10}} (u1);

  % --- From u2 ---
  \path
    (u2) edge node[lab, right, pos=0.55,align=center]{\edgelab{b}{}{x,y}{1000}\\\edgelab{h}{}{}{-800}} (u0)
    (u2) edge[loop above] node[lab,align=center]{\edgelab{!a\&!b\&!h}{}{}{-10}\\ \edgelab{a}{}{}{-10}} (u2);

  % --- From u3 ---
  \path
    (u3) edge node[lab, above, pos=0.5]{\edgelab{a}{}{}{800}} node[lab, below, pos=0.5]{\edgelab{h}{}{}{-800}} (u0)
    (u3) edge[loop below] node[lab, align=center]{\edgelab{b}{}{}{-10}\\ \edgelab{!a\&!b\&!h}{y>1}{y}{-10}} (u3);
\end{tikzpicture}
}
\caption{TRM5 example on Frozen Lake}
\label{TRM:cont-is-better}
\end{figure}

\begin{figure}[h]
\centering
\scalebox{0.6}{%
\begin{minipage}{\linewidth}
\centering
\includegraphics[width=.9\linewidth]{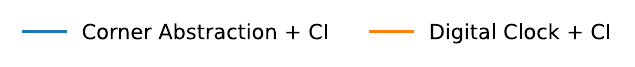}\\[.6ex]
\includegraphics[width=.49\linewidth]{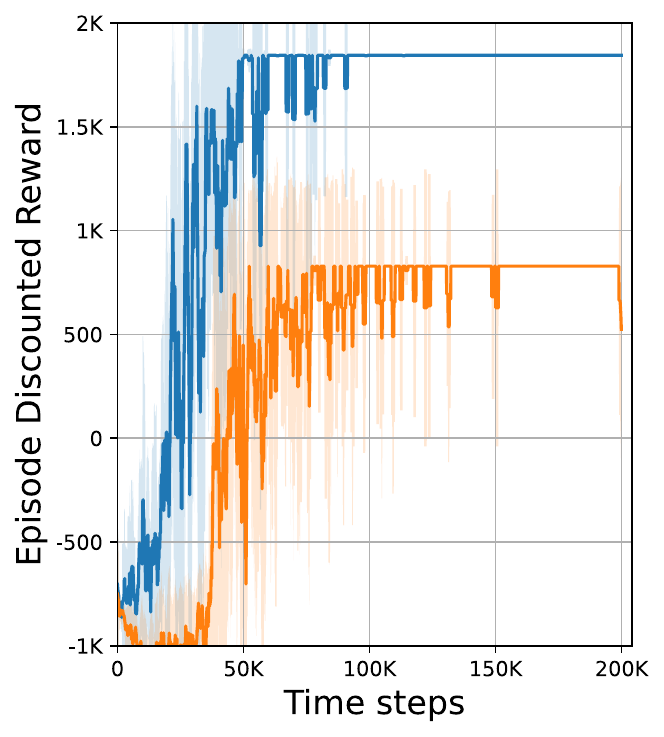}%
\hfill
\includegraphics[width=.49\linewidth]{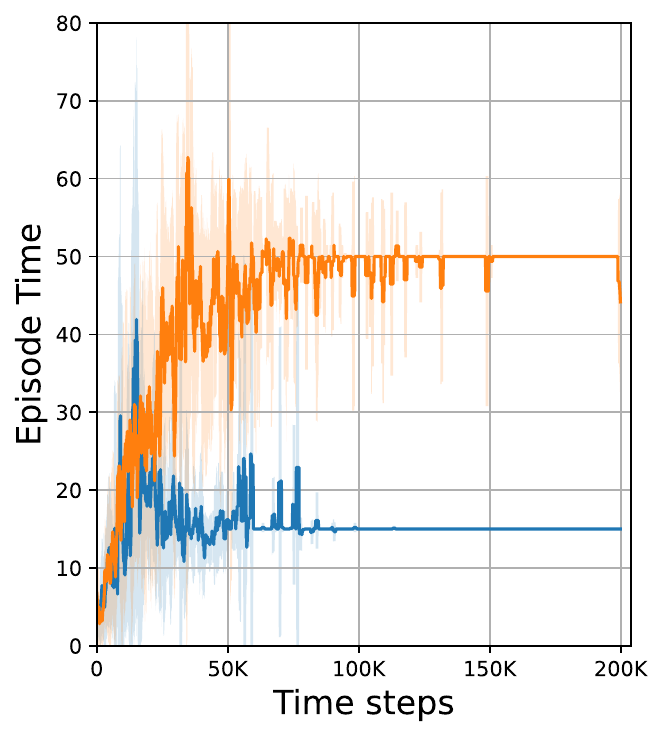}
\end{minipage}%
}
\caption{Pronounced performance difference between corner-point abstraction and digital clock abstraction for a TRM specification (TRM5, shown above) that requires precise timing.}
\label{plot:cont-is-better}
\end{figure}

\subsection*{RQ3: Scale of experiments}

\begin{table*}[h]
\centering
\footnotesize
\setlength{\tabcolsep}{3.8pt}
\renewcommand{\arraystretch}{1.18}

\begin{tabular}{l c c c c c c c}
\toprule
\textbf{Setting} &
\textbf{MDP (state-size)} &
\textbf{TRM (state-size)} &
\textbf{Product Size} &
\textbf{Explored States} &
\textbf{Learning Time (s)} \\
\midrule

% ===================== TAXI =====================
\multirow{4}{*}{Digital-clock}  & Taxi (500) & TRM1 (6) & 51{,}000      & 1{,}050 & 184.13 \\
                               &  Taxi (500) & TRM3 (5) & 127{,}500     & 910     & 50.14  \\
  & Frozen Lake (64) & TRM2 (4) & 13{,}056      & 1{,}693 & 40.32 \\                                 
&  Frozen Lake (64) & TRM4 (4) & 768          & 174     & 35.23 \\

\midrule

% ===================== FROZEN LAKE =====================
% &                                 & TRM5 & 64  & 4 & 5{,}376      & 268     & 33.13 \\
\multirow{4}{*}{Real-time}      & Taxi (500) & TRM1 (6) & 1{,}402{,}500 & 3{,}148 & 356.02 \\
                                 &  Taxi (500) & TRM3 (5) & 23{,}358{,}000& 2{,}990 & 84.21  \\
    & Frozen Lake (64) & TRM2 (4) &  2{,}242{,}368 & 5{,}763 & 84.53 \\
                                &  Frozen Lake (64) & TRM4 (4) &  2{,}880      & 247     & 45.23 \\
% &                                 & TRM5 & 64  & 4 & 358{,}848    & 2{,}428 & 183.01 \\
\bottomrule
\end{tabular}

\caption{Training statistics for the experiments.}
\label{tab:scalability-stats}
\end{table*}

We report in Table~\ref{tab:scalability-stats} the scale of our experiments in terms of the size of the underlying MDP, the size of the TRM, the size of the induced cross-product MDP, and the number of explored states during learning and the total learning time for each experiment.
For each environment and TRM specification, we report the size of the underlying MDP, the size of the corresponding TRM, and the size of the induced cross-product MDP, together with the total learning time.
In the digital-clock setting, the cross-product size accounts for all possible discrete clock valuations, whereas in the real-time setting, the size reflects the incorporation of corner-region abstraction.
We also report the number of explored states encountered during learning.

First, these statistics indicate that the scale of our experiments is on par with the contemporary works on RL with finite-state machines, e.g.,~\cite{DBLP:conf/ijcai/FalahG025,DBLP:conf/ecai/HahnPSS0W23,DBLP:conf/aaai/CorazzaGN22}.

Further, we note that the number of explored states is orders of magnitude smaller than the size of the corresponding cross-product MDPs.
This gap arises because a large fraction of product states are unreachable from the initial state under any feasible policy.
These results highlight a key advantage of our sampling-based reinforcement learning framework: effective policy learning without exhaustive exploration of the full (timed) product state space.
In contrast, many deductive planning and control approaches for timed automata typically require explicit construction of the product system or carefully designed heuristics to ensure sufficient coverage of the entire state space. The consistently moderate learning times across the experiments further support the practical scalability of the proposed framework.

\section{Technical discussion with selected related works}
\label{sec:technical-comparison-related-works}
We expand upon the technical comparison with key related works by providing more details.
Also, we clarify which related works can be fairly compared against our setting.

\subsection{
Comparison with works in Priced Timed Automata (PTAs)}
\begin{itemize}
\item These models have been primarily used in control and planning~\cite{DBLP:conf/cav/BehrmannCDFLL07,DBLP:conf/hybrid/BouyerBL04,DBLP:journals/fmsd/BouyerBL08,DBLP:conf/fsttcs/BouyerCFL04}, where knowledge of the underlying model is assumed. Consequently, these works develop deductive methods for using such specifications, as opposed to the sampling-based statistical RL methods in a model-free setting. Learning in a model-free setting has the advantage of not storing the entire model of the MDP and the TRM. This advantage can be seen from the scalability table, Table~\ref{tab:scalability-stats}, where the explored space is significantly less than the possible cross-product space.
So, no comparable PTA baseline exists in the model-free setting.

\item The reward computation for PTAs is usually based on cumulative sums~\cite{DBLP:conf/fsttcs/BouyerCFL04} or ratios~\cite{DBLP:journals/fmsd/BouyerBL08}. In contrast, TRMs consider discounted sums, which are standard in RL and require different analysis. Moreover, TRMs can incorporate Markovian reward functions on states and transitions (e.g., cost c in Figure~\ref{fig:small_example}).
\end{itemize}

\subsection{Comparison with works in Reward Machines (RMs)}
\begin{itemize}
    \item TRMs add expressive power to RMs~\cite{DBLP:phd/ca/Icarte22} in the standard RL setting by adding clock constraints, with RMs being a strict subclass of TRMs. Consequently, TRMs also have an empirical advantage over RMs in handling time-sensitive requirements, as shown by our experiments.
    \item Our work differs from RMs interpreted on CTMDPs~\cite{DBLP:conf/ijcai/FalahG025} as these involve specifications with no timing constraints (e.g., guards or interval bounds). Thus, our novelty of adding clocks to RM with suitable constraints requires techniques from the TA literature, such as region/zone abstractions, which are not required in~\cite{DBLP:conf/ijcai/FalahG025}. 
\end{itemize}

\subsection{Comparison with works in Duration Calculus (DC)}
\begin{itemize}
\item Duration Calculus~\cite{DBLP:conf/rtss/DoleGKKT21} falls under declarative logical formalisms (e.g., LTL, MTL), whereas our work is based on reward-based automata formalisms (numerous RM formalisms). These two classes of formalisms serve different purposes as specifications in RL. While declarative formalisms are closer to natural language, they do not offer the same fine-grained control over costs and rewards as reward-based automata do.
\item Methodologically, RL with declarative formalisms~\cite{DBLP:conf/icra/Bozkurt0ZP20,DBLP:conf/rtss/DoleGKKT21,DBLP:conf/cdc/HasanbeigKAKPL19} focus on effectively converting the logical specifications to automata monitors (where reward 1 denotes accept, 0 reject) that track the satisfaction of the specification. One would then use RL techniques developed for RM~\cite{DBLP:conf/ijcai/ShaoK23} itself for the monitors. In this sense, our techniques for TRM improve RL for timed specifications by leveraging various abstractions and counterfactuals.

\item RL using automata monitors is known to suffer from the inability to learn long-horizon temporally extended tasks. We empirically demonstrate this in a simple scenario where learning with an automaton monitor fails to satisfy the necessary guards to complete a long-horizon task, whereas our TRM does.

We consider a task specification in the taxi domain which requires the taxi to (i) pick up the passenger after at least 12 time units, (ii) then visit the green zone within the next 15 time units, and (iii) then drop the passenger at the destination within the next 15 time units.
We first encode this specification as a timed automaton monitor that assigns a reward of~1 only upon reaching the accepting state corresponding to successful completion of all three timing constraints, and a reward of~0 otherwise.
This represents a purely declarative specification with sparse terminal rewards.

In contrast, we construct an equivalent timed reward machine in which intermediate rewards are provided upon progress toward satisfying individual timing constraints.
Such intermediate rewards are naturally supported in the TRM formalism used in our framework and allow the agent to receive incremental feedback during learning.

The TRM representing the timed behavior described above is presented in Fig.~\ref{trm:comparison_with_DC}. Also note that, one can simply modify this TRM to get the corresponding timed automaton monitor: replace the final reward (on the transition from $u_4$ to $u_0$) with $1$ and everything else with $0$.

\begin{figure*}[h]
\centering
% \resizebox{\columnwidth}{!}{%
\begin{tikzpicture}[
  ->, >=stealth, on grid, auto, font=\footnotesize,
  node distance=4.2cm,
  state/.style={rectangle, rounded corners, draw,
                minimum width=12mm, minimum height=6mm},
  finalstate/.style={state, double, double distance=1pt},
  sink/.style={state, fill=gray!10},
  lab/.style={inner sep=1pt, align=center},
  every loop/.style={looseness=6},
  shorten >=1pt, shorten <=1pt
]
  % Helpers
  \newcommand{\stlab}[2]{$u_{#1}$,\,#2}
  \newcommand{\edgelab}[4]{%
    \tiny
    \texttt{#1},\ %
    \if\relax\detokenize{#2}\relax $\top$ \else $#2$ \fi,\ %
    \if\relax\detokenize{#3}\relax $\varnothing$ \else $\{#3\}$ \fi,\ %
    #4%
  }

  % Nodes
  \node[state, initial, initial text=] (u1) at (0,0)    {\stlab{1}{-20}};
  \node[state]                         (u2) at (4.2,0)  {\stlab{2}{-20}};
  \node[state]                         (u3) at (9.5,0)    {\stlab{3}{-20}};
  \node[state]                    (u4) at (9.5,-1.6) {\stlab{4}{-20}};
  \node[sink]                          (u0) at (4.2,-1.6) {\stlab{0}{-20}};

  % --- Transitions from u1 ---
  \path
    (u1) edge node[lab, above, pos=0.48]
      {\edgelab{in\_taxi}{x>12}{x}{200}} (u2)
    (u1) edge[loop above] node[lab]
      {\edgelab{!in\_taxi}{}{}{-5}} (u1);

  % --- Transitions from u2 ---
  \path
    (u2) edge node[lab, above, pos=0.52]
      {\edgelab{at\_green\&in\_taxi}{x \le 15}{x}{400}} (u3)
    (u2) edge[loop above] node[lab]
      {\edgelab{!at\_green\&in\_taxi}{}{}{-5}} (u2)
    (u2) edge node[lab, left, pos=0.5]
      {\edgelab{!in\_taxi}{}{}{-5}} (u0);

  % --- Transitions from u3 ---
  \path
    (u3) edge node[lab, right, pos=0.48, align=center]
      {\tiny\texttt{at\_dest},\\[-1mm]
       \tiny$x\le 15$,\\[-1mm]
       \tiny$\varnothing$, \tiny600} (u4)
    (u3) edge[loop above] node[lab]
      {\edgelab{!at\_dest\&in\_taxi}{}{}{-5}} (u3)
    (u3) edge node[lab, above, pos=0.55, sloped]
      {\edgelab{!at\_dest\&!in\_taxi}{}{}{-5}} (u0);

  % --- Transitions from u4 ---
  \path
    (u4) edge node[lab, align=center]
      { \edgelab{drop\_off}{}{}{800}\\[-1mm]
        \edgelab{!drop\_off}{}{}{-5} }
      (u0);
\end{tikzpicture}
% }
\caption{TRM6 for comparison between TRMs with declarative models}
\label{trm:comparison_with_DC}
\end{figure*}

Our results show that, under the declarative timed automaton monitor, the agent fails to reliably learn policies that satisfy the full timed specification.
In contrast, when guided by the TRM-based reward structure, the agent progressively learns to satisfy each timing constraint during training and ultimately succeeds in completing the entire task.
The number of guards that each of these models satisfy is shown in Fig.~\ref{fig:guards-satisfaction}.
This demonstrates that, in timed settings, TRMs provide a more effective learning signal than purely declarative models, leading to improved learnability without altering the underlying specification.

\begin{figure}[h]
\centering
\includegraphics[width=0.5\linewidth]{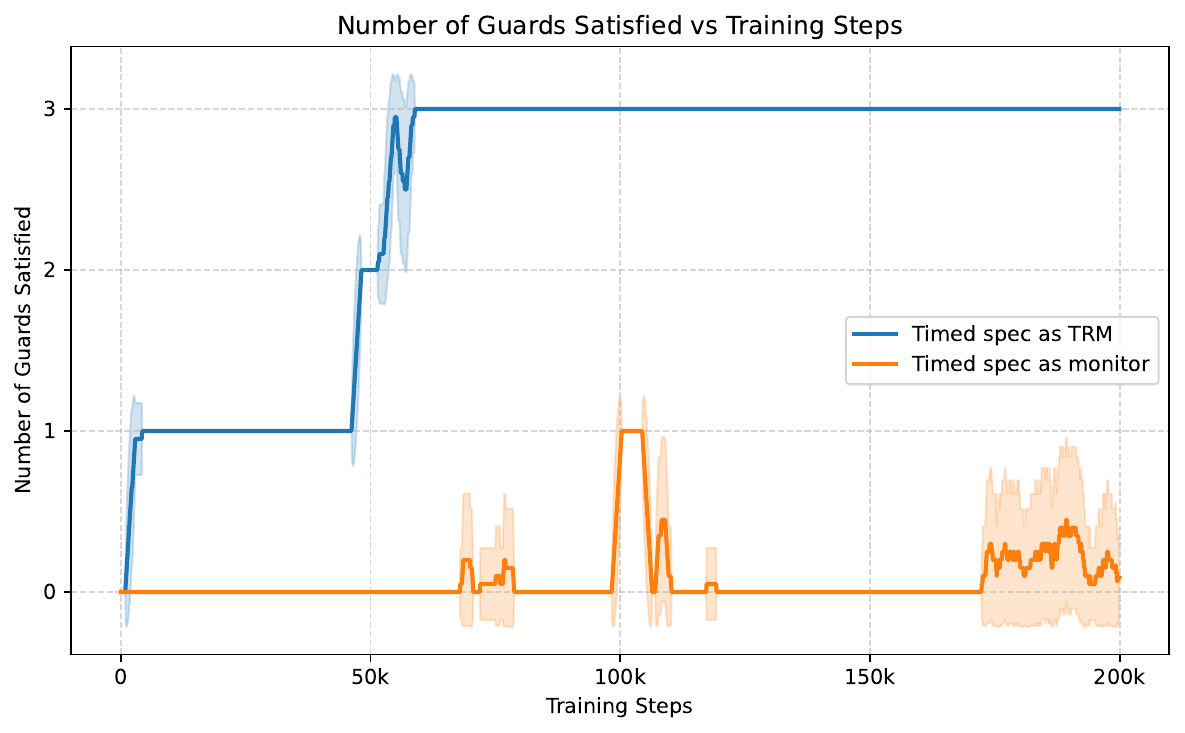}
\caption{Guard satisfaction over training steps in the Taxi environment.}
\label{fig:guards-satisfaction}
\end{figure}

\item Since the abilities, purpose, and features of declarative formalisms (e.g., LTL, DC) are substantially different from the reward-based formalisms (RMs, omega-regular RMs), it is uncommon in this literature to compare these different classes (e.g., see~\cite{DBLP:conf/ecai/HahnPSS0W23,IcarteKVM18}).
\end{itemize}

\section{Conclusion}
We studied model-free RL for \emph{Timed Reward Machines} (TRMs), a formalism that extends reward machines with explicit timing constraints. We interpreted TRMs over MDPs under digital and real-time semantics and devised abstractions for efficient learning. Our experiments with non-trivial timed specifications show that TRMs enable learning policies with delays for maximizing rewards.

This work represents a step toward improving time-sensitive reward specification in RL, with numerous avenues ahead. One can apply TRMs to continuous-time Markov models~\cite{DBLP:conf/ijcai/FalahG025}, which better capture rate-based timing; adapt deep continuous RL (e.g., TD3~\cite{DBLP:conf/icml/FujimotoHM18}) to continuous-time; and incorporate guidance from priced zones~\cite{BehrmannLR04} to improve exploration of TRM objectives.

\paragraph{Acknowledgements.}
Rajarshi Roy, David Parker and Marta Kwiatkowska received funding from the ERC under the European Union’s Horizon 2020 research and innovation programme (grant agreement No.834115, FUN2MODEL).
Rajarshi Roy was partially funded
by the European Union (RobustifAI project, ID 101212818).
Views and opinions expressed are however those of the
author(s) only and do not necessarily reflect those of the
European Union or the European Health and Digital Executive Agency (HADEA). Neither the European Union nor
the granting authority can be held responsible for them.
Anirban Majumdar was supported by the Department of Atomic Energy, Government of
India, under project no. RTI4014.
We also thank Anne-Kathrin Schmuck for insightful discussions.

%%%%%%%%%%%%%%%%%%%%%%%%%%%%%%%%%%%%%%%%%%%%%%%%%%%%%%%%%%%%%%%%%%%%%%%%
\bibliographystyle{splncs04} 
\bibliography{bib}

%%%%%%%%%%%%%%%%%%%%%%%%%%%%%%%%%%%%%%%%%%%%%%%%%%%%%%%%%%%%%%%%%%%%%%%%
%\newpage
%\appendix
%\input{appendix}
%%%%%%%%%%%%%%%%%%%%%%%%%%%%%%%%%%%%%%%%%%%%%%%%%%%%%%%%%%%%%%%%%%%%%%%%

\end{document}